\documentclass{new_tlp}
\usepackage{amsmath,amsfonts}
\usepackage[obeyFinal]{todonotes}
\usepackage{booktabs}
\usepackage{xspace}
\usepackage[ruled, vlined, linesnumbered]{algorithm2e}
\usepackage{enumitem}
\newtheorem{theorem}{Theorem}
\newtheorem{definition}[theorem]{Definition}
\newtheorem{corollary}[theorem]{Corollary}
\newtheorem{lemma}[theorem]{Lemma}
\newtheorem{example}[theorem]{Example}
\newtheorem{proposition}[theorem]{Proposition}

\newcommand{\DLL}{\text{DL-Lite}\xspace}
\newcommand{\DLLR}{\text{DL-Lite}\ensuremath{_R}\xspace}
\newcommand{\DLLH}{\text{DL-Lite}\ensuremath{_\text{Horn}}\xspace}

\newcommand{\NC}{\ensuremath{N_C}\xspace}
\newcommand{\NR}{\ensuremath{N_R}\xspace}
\newcommand{\NI}{\ensuremath{N_I}\xspace}
\newcommand{\NN}{\ensuremath{N_N}\xspace}
\newcommand{\NV}{\ensuremath{N_V}\xspace}

\newcommand{\Amc}{\ensuremath{\mathcal{A}}\xspace}
\newcommand{\Dmc}{\ensuremath{\mathcal{D}}\xspace}
\newcommand{\Imc}{\ensuremath{\mathcal{I}}\xspace}
\newcommand{\Nmc}{\ensuremath{\mathcal{N}}\xspace}
\newcommand{\Omc}{\ensuremath{\mathcal{O}}\xspace}
\newcommand{\Tmc}{\ensuremath{\mathcal{T}}\xspace}
\newcommand{\Vmc}{\ensuremath{\mathcal{V}}\xspace}
\newcommand{\Aexa}{\ensuremath{\mathcal{A}_\textsf{exa}}\xspace}
\newcommand{\Oexa}{\ensuremath{\mathcal{O}_\textsf{exa}}\xspace}
\newcommand{\Texa}{\ensuremath{\mathcal{T}_\textsf{exa}}\xspace}
\newcommand{\Aexo}{\ensuremath{\mathcal{A}_0}\xspace}
\newcommand{\Oexo}{\ensuremath{\mathcal{O}_0}\xspace}
\newcommand{\Texo}{\ensuremath{\mathcal{T}_0}\xspace}
\newcommand{\Aexb}{\ensuremath{\mathcal{A}_\textsf{exb}}\xspace}
\newcommand{\Oexb}{\ensuremath{\mathcal{O}_\textsf{exb}}\xspace}
\newcommand{\Texb}{\ensuremath{\mathcal{T}_\textsf{exb}}\xspace}
\newcommand{\Aexc}{\ensuremath{\mathcal{A}_2}\xspace}
\newcommand{\Oexc}{\ensuremath{\mathcal{O}_2}\xspace}
\newcommand{\Texc}{\ensuremath{\mathcal{T}_2}\xspace}

\newcommand{\va}{\ensuremath{\mathbf{a}}\xspace}
\newcommand{\vx}{\ensuremath{\mathbf{x}}\xspace}
\newcommand{\vy}{\ensuremath{\mathbf{y}}\xspace}

\newcommand{\ans}{\ensuremath{\mathsf{ans}}\xspace}
\newcommand{\At}{\ensuremath{\mathsf{At}}\xspace}
\newcommand{\var}{\ensuremath{\mathsf{var}}\xspace}

\newcommand{\AC}{\textsc{AC}\ensuremath{^0}\xspace}
\newcommand{\LS}{\textsc{LogSpace}\xspace}
\newcommand{\NLS}{\textsc{NLogSpace}\xspace}

\newcommand{\NP}{\textsc{NP}\xspace}

\newcommand{\Ican}{\ensuremath{{\mathcal{I}_\mathsf{can}}}\xspace}
\newcommand{\Icanp}{\ensuremath{{\mathcal{I}'_\mathsf{can}}}\xspace}

\newcommand{\ax}[2][1]{\ensuremath{\left<#2,#1\right>}\xspace}

\newcommand{\uv}{\rule{2mm}{1pt}\xspace}

\begin{document}
%
\title{Answering Fuzzy Queries over Fuzzy DL-Lite Ontologies}
%
%
\author[Gabriella Pasi, Rafael Pe\~naloza]
	{Gabriella PASI,
	  Rafael PE\~NALOZA \\
	University of Milano-Bicocca, Italy \\
	\email{\{gabriella.pasi,rafael.penaloza\}@unimib.it}}
%
%
%
\maketitle             
\begin{abstract} 
A prominent problem in knowledge representation is how to answer queries taking into account also the implicit consequences
of an ontology representing domain knowledge. While this problem has been widely studied within the realm of description
logic ontologies, it has been surprisingly neglected within the context of vague or imprecise knowledge, particularly from
the point of view of mathematical fuzzy logic. In this paper we study the problem of answering conjunctive queries and
threshold queries w.r.t.\ ontologies in fuzzy \DLL. Specifically, we show through a rewriting approach that threshold query
answering w.r.t.\ consistent ontologies remains in \AC in data complexity, but that conjunctive query answering is highly
dependent on the selected triangular norm, which has an impact on the underlying semantics. For the idempodent G\"odel t-norm, 
we provide an effective method based on a reduction to the classical case. This paper is under consideration in Theory and Practice of 
Logic Programming (TPLP).
\end{abstract}
\section{Introduction}

Description logics (DLs) \cite{dlhandbook} are a well-known and widely used family of knowledge representation formalisms 
that, thanks to their clear syntax and formal semantics, have been used to represent and deal with the knowledge of 
various representation domains. Among the many members of this family, a subfamily of languages with a limited expressivity, known
as the \DLL family \cite{dl-lite} has a prominent role. In fact, simplifying a bit, \DLL was originally designed with the goal of including 
background knowledge within the task of answering queries, and avoiding the need for an explicit enumeration of all the facts
that are implicitly implied by the domain knowledge.

Consider for example a touristic scenario, which includes information about museums, monuments, restaurants, and pubs. 
Knowing that museums and monuments are touristic attractions, and that restaurants and pubs are eateries, one can
immediately deduce that the modern art museum and the peace monument are touristic attractions, and that the Irish pub is
an eatery, without having to make this knowledge explicit. A user may thus ask for e.g., a \emph{tourist attraction that contains
an eatery}. Using classical query answering techniques \cite{OrSi-RW12}, all attractions that satisfy this requirement can be 
efficiently retrieved.

Being based on classical logic, DLs in general and \DLL in particular are unable to handle imprecise or vague knowledge effectively.
In our touristic scenario, for instance, we may want to extend the knowledge with some additional properties of the objects
of interest. For example, a tourist in a hurry may want to visit the \emph{popular} attractions first; or a backpacker on a budget
may be more interested in finding \emph{cheap} eateries. Note that \emph{cheap} and \emph{popular} are two vague
notions that do not allow for any precise definition. In a simplistic scenario, cheapness may be defined in terms of the mean
cost for a meal, but even then, it is impossible to specify a precise price-point where an eatery stops being cheap; moreover this is 
also a subjective notion. The case
of popularity is even worse, as there is no obvious proxy for it.

To solve this issue, fuzzy extensions of DLs have been widely studied; see for example 
\cite{Borgwardt:PhD,BoPe-SUM17,BCE+-15,LuSt08,Cera:PhD} and references therein. In essence, fuzzy
logic \cite{Haje98} extends classical logic by allowing truth degrees in the interval $[0,1]$ for 
the propositions that contain fuzzy (or vague) predicates.
One can thus say, e.g., that the modern art museum is popular to a degree
of $0.8$ meaning, intuitively, that it is popular, but more popular attractions may exist. 

Interestingly, although fuzzy DLs and their reasoning services have been widely studied, the task of answering queries 
based on fuzzy ontologies has been mostly ignored. Most of the earlier work from this point of view was carried out by
Straccia and Pan.
Specifically, Straccia \cite{Stra-JELIA06} studied the problem of computing the answers with
highest degree on a query w.r.t.\ some background knowledge. This was followed by Pan et al. \cite{PSST-DL07}, who 
considered more complex queries to be answered. 
While from some perspective these works seem to cover the whole area of query answering, 
they were based on the so-called Zadeh semantics, which does not have adequate properties from a mathematical logic
point of view \cite{Haje98}. 
Another limitation of all these approaches is that they allowed only the facts in the ontology to be graded, but
restricted the terminological knowledge to be crisp (i.e., hold fully).
Other work considering query answering in fuzzy DLs 
includes~\cite{Stra-IS12}, where the $k$ answers with the highest degree are retrieved. This latter work is closer to our approach
but has several limitations. Perhaps the most obvious is that its semantics follows a closed-world assumption, even in the
case of background knowledge. In addition, background knowledge is interpreted as a \emph{rule}, where the degrees
of the body of an axiom define the degree of the head, but knowledge about the head cannot be used to infer knowledge about
the body. We, in change, use the open world assumption, as typical in knowledge representation, and use the logical interpretation
of axioms.

Later on, Turhan and Mailis studied the problem of query answering w.r.t.\ background knowledge from the point of view of fuzzy 
logic \cite{Haje98}, where the
semantics are based on the properties of continuous triangular norms \cite{KlMP00}. They developed
a technique for computing the satisfaction degrees of conjunctive queries when the semantics were based on the G\"odel 
t-norm \cite{MaTu-JIST14}.
This technique, which is based on the construction of a classical query, was later implemented and shown to be effective in
\cite{MaTZ-DL15}. However, it still suffered from two main drawbacks: (i) it was only capable to handle the idempotent
(G\"odel)
t-norm, and (ii) terminological knowledge had to still be precise, allowing no graded axioms. The latter condition is essential for 
the correctness of their
approach: their reduction is unable to keep track of the degrees used by the terminological axioms, as this would require
an unbounded memory use.

In this paper, we study the problem of query answering w.r.t.\ \DLL ontologies, filling out the gaps left by the existing work.
To be more explicit, our work is the first to consider adequate semantics from the mathematical fuzzy logic point of view, alongside
graded axioms stating vague knowledge beyond just vague data.
We start by considering the kind of conjunctive queries studied by Turhan and Mailis, but allowing 
the association of numeric degrees 
also in the TBox. Interestingly, although this is a generalization of the previously studied setting, we are able to develop
a much simpler method, which does not rely on rewriting, but rather on a reduction to a classical query answering scenario.
The method is based on the idea of \emph{cut ontologies}, where all knowledge holding to a low degree is ignored. Hence, we 
obtain
a more robust and easier to maintain approach than previous work. Still considering the G\"odel t-norm, we considered the
case of threshold queries, also left open in previous work, in which every conjunct in a query is assigned a different degree.
In this case, a direct reduction to classical query answering does not work, but we were able to adapt the classical rewriting
methods to handle the degrees effectively. 

The final part of the paper considers other t-norms as the underlying semantics for the fuzzy constructors. In this case,
we show through several examples that conjunctive queries cannot be easily handled, but we identify some special cases where
queries can be effectively answered. On the other hand, we show that we can still apply the rewriting technique to answer
threshold queries, even for non-idempotent t-norms. This is a surprising result because in the idempotent scenario threshold
queries are a generalization of conjunctive queries. 

Some of the results in this paper were previously published in \cite{PaPe-RR20}. In addition to full proofs, deeper explanations,
and examples, here we extend that previous work by handling threshold queries, including the full rewriting technique from
Section \ref{sec:tq}. We also provide better results for non-idempotent t-norms, and highlight some of the problems of combining
conjunctions and non-idempotent t-norms in the appendix.

\section{Preliminaries}

We briefly introduce the syntax and semantics of fuzzy \DLLR and other related notions that will be important for this paper.
Let \NC, \NR, and \NI be three mutually disjoint sets whose elements are called \emph{concept names}, \emph{role names},
and \emph{individual names}, respectively. The sets of \DLLR\emph{concepts} and \emph{roles} are built through the 
grammar rules:
\begin{align*}
B::={}& A\mid \exists Q & C::={}&B\mid\neg B \\
Q::={}& P\mid P^- & R::={}&Q\mid\neg Q
\end{align*}
where $A\in\NC$ and $P\in\NR$. Concepts of the form $B$ and roles of the form $Q$ are called \emph{basic}, and all others
are called \emph{general}. 

\begin{definition}[ontology]
A \emph{fuzzy \DLLR TBox} is a finite set of \emph{fuzzy axioms} of the form 
$\left<B\sqsubseteq C,d\right>$ and $\left<Q\sqsubseteq R,d\right>$, 
where $d$ is a number in $[0,1]$. An axiom is \emph{positive} if it does not have negation on its right-hand side and \emph{negative}
otherwise. 
A \emph{fuzzy \DLLR ABox} is a finite set of \emph{fuzzy assertions} of the form 
$\left<B(a),d\right>$ and $\left<P(a,b),d\right>$, where $a,b\in\NI$.
A \emph{fuzzy \DLLR ontology} is a pair of the form $\Omc=(\Tmc,\Amc)$ where \Tmc is a TBox and \Amc is an ABox.
\end{definition}
Note that negations can never occur on the left-hand side of an axiom.
In the remainer of this paper, we will mostly exclude the qualifiers ``fuzzy,'' and ``\DLL'' and simply refer to axioms, ontologies,
etc. 

The semantics of fuzzy \DLLR is based on fuzzy interpretations, which provide a \emph{membership degree} or 
for objects belonging to the different concept and role names. Formally, following the basics of classical
description logics, concept names are interpreted as fuzzy unary relations, and role names are interpreted as fuzzy binary 
relations.
To fully define this semantics in the presence of other constructors according to 
fuzzy logic, we need the notion of a triangular norm (or \emph{t-norm} for short).

\begin{definition}[t-norm]
A \emph{t-norm} $\otimes$ is a binary operator over the real interval $[0,1]$ that is commutative, associative, monotonic, 
and has $1$ as the neutral element; i.e., $1\otimes x=x$ for all $x\in[0,1]$ \cite{KlMP00}. 
\end{definition}
Triangular norms are used to generalize the logical conjunction to handle truth degrees that take values from the interval $[0,1]$. 
Every continuous t-norm defines a 
unique \emph{residuum} $\Rightarrow$ where $f\otimes d\le e$ iff $f\le d\Rightarrow e$. The residuum interprets implications. 
With the help of this operation, it is also possible to interpret other logical operators such as negation ($\ominus d:=d\Rightarrow 0$).
The three basic continuous t-norms are the \emph{G\"odel}, \emph{\L ukasiewicz}, and \emph{product} t-norms, which are 
defined, with their residua and negations in Table~\ref{tab:tnorm}.
\begin{table}[tb]
\caption{The three fundamental continuous t-norms and related operations}
\label{tab:tnorm}
\centering
\begin{tabular}{@{}l@{\qquad}l@{\qquad}l@{\qquad}l@{}}
\toprule
Name & $d \otimes e$ & $d \Rightarrow e$ & $\ominus d$ \\
\midrule
G\"odel & $\min\{d,e\}$ & $\begin{cases}1&d\le e\\ e&\text{otherwise}\end{cases}$ & $\begin{cases}1&d=0\\ 0&\text{otherwise}\end{cases}$\\
\L ukasiewicz & $\max\{d+e-1,0\}$ & $\min\{1-d+e,1\}$ & $1-d$\\
product & $d\cdot e$ & $\begin{cases}1&d\le e\\ e/d&\text{otherwise}\end{cases}$ & $\begin{cases}1&d=0\\ 0&\text{otherwise}\end{cases}$\\
\bottomrule
\end{tabular}
\end{table}
These t-norms are the ``fundamental'' ones in the sense that every other continuous t\mbox{-}norm is isomorphic to the ordinal sum
of copies of them \cite{Haje98,MoSh-AM57}. Hence, as usual, we focus our study on these three t-norms.

Note that the residuum always satisfies that $d\Rightarrow e=1$ iff $d\le e$, and that in the G\"odel and product t-norms the
negation is annihilating in the sense that it maps to 0 any positive value, while the negation of 0 is 1. 
In particular, this means that the negation is not \emph{involutive}; that is, $\ominus\ominus d\not=d$ in general. 
In contrast, the negation operator for the \L ukasiewicz t-norm is involutive. In addition, the \L ukasiewicz t-norm is the only
t-norm (up to isomorphism) with the property that for every $x\in (0,1)$ there exists a $y\in (0,1)$ such that $x\otimes y=0$. 
Specifically, this $y$ is $1-x$. In other words, the \L ukasiewicz t-norm is \emph{nilpotent}.
From now on, unless 
specified explicitly otherwise, we assume that we have an arbitrary, but fixed, t-norm $\otimes$ which underlies the operators 
used. When the t\mbox{-}norm becomes relevant in the following sections, we will often use G, $\Pi$, and \L{} as prefixes 
to express that the underlying t-norm is G\"odel, product, or \L ukasiewicz, respectively, as usual in the literature.

We can now formally define the semantics of the logic. An \emph{interpretation} is a pair $\Imc=(\Delta^\Imc,\cdot^\Imc)$,
where $\Delta^\Imc$ is a non-empty set called the \emph{domain}, and $\cdot^\Imc$ is the \emph{interpretation function}
which maps: (i) every individual name $a\in\NI$ to an element $a^\Imc\in\Delta^\Imc$; (ii) every concept name $A\in\NC$ to a 
function $A^\Imc:\Delta^\Imc\to[0,1]$; and (iii) every role name $P\in\NR$ to a function 
$P^\Imc:\Delta^\Imc\times\Delta^\Imc\to[0,1]$. That is, concept names are interpreted as fuzzy unary relations and
role names are interpreted as fuzzy binary relations over $\Delta^\Imc$. The interpretation function is extended to other
constructors with the help of the t-norm operators as follows. For every $\delta,\eta\in\Delta^\Imc$,
\begin{align*}
(\exists Q)^\Imc(\delta) := {} & \sup_{\delta'\in\Delta^\Imc}Q^\Imc(\delta,\delta') &
(\neg B)^\Imc(\delta) := & \ominus B^\Imc(\delta) &
(\top)^\Imc(\delta) := & 1 \\
(P^-)^\Imc(\delta,\eta) := {} & P^\Imc(\eta,\delta) &
(\neg Q)^\Imc(\delta,\eta) := & \ominus Q^\Imc(\delta,\eta) &
\end{align*}
The interpretation \Imc \emph{satisfies} the axiom 
\begin{itemize}
\item $\left<B\sqsubseteq C,d\right>$ iff $B^\Imc(\delta)\Rightarrow C^\Imc(\delta)\ge d$ holds for every $\delta\in\Delta^\Imc$; and
\item $\left<Q\sqsubseteq R,d\right>$ iff $Q^\Imc(\delta,\eta)\Rightarrow R^\Imc(\delta,\eta)\ge d$ holds for every 
	$\delta,\eta\in\Delta^\Imc$
\end{itemize}
It is a \emph{model} of the TBox \Tmc if it satisfies all axioms in \Tmc.
\Imc \emph{satisfies} the assertion
\begin{itemize}
\item $\left<B(a),d\right>$ iff $B^\Imc(a^\Imc)\ge d$;
\item $\left<P(a,b),d\right>$ iff $P^\Imc(a^\Imc,b^\Imc)\ge d$.
\end{itemize}
It is a \emph{model} of the ABox \Amc if it satisfies all axioms in \Amc, and it is a \emph{model} of the ontology
$\Omc=(\Tmc,\Amc)$ if it is a model of \Tmc and of \Amc.

We note that the classical notion of \DLLR \cite{dl-lite} is a special case of fuzzy \DLLR, where all the axioms and assertions hold
with degree 1. In that case, it suffices to consider interpretations which map all elements to $\{0,1\}$ representing the classical
truth values. When speaking of classical ontologies, we remove the degree and assume it implicitly to be 1.

\begin{example}
\label{exa:run}
Consider an ontology $\Oexa=(\Texa,\Aexa)$ representing some knowledge about a touristic location. The TBox 
\begin{align*}
\Texa =\{ & 
	\ax{\textsf{Monument} \sqsubseteq \textsf{TouristAttraction}}, \quad 
		\ax{\textsf{Museum} \sqsubseteq \textsf{TouristAttraction}}, \\ &
	\ax{\textsf{Pub} \sqsubseteq \textsf{Eatery}}, \quad \ax{\textsf{Restaurant} \sqsubseteq \textsf{Eatery}}, 
		\quad \ax{\textsf{locIn}\sqsubseteq \textsf{Near}}\\ &
	\ax[0.6]{\textsf{Museum} \sqsubseteq \textsf{Popular}}, \quad 
		\ax[0.5]{\exists \textsf{locIn}\sqsubseteq \neg \textsf{Cheap}} & \}
\end{align*}
defines some notions about eateries and tourist attractions, including some vague notions in the last two axioms. 
For example, it expresses that museums are popular (with a degree at least 0.6), and that services located at some attraction are 
not cheap (with degree at least 0.5). 
The ABox
\begin{align*}
\Aexa = \{ &
	\ax{\textsf{Monument}(\textsf{peace})}, \quad \ax{\textsf{Monument}(\textsf{love})}, \\ &
	\ax{\textsf{Museum}(\textsf{modernArt})}, \quad \ax{\textsf{Museum}(\textsf{contArt})}, \quad
		\ax{\textsf{Museum}(\textsf{comic})}, \\ &
	\ax{\textsf{Restaurant}(\textsf{sioux})}, \quad \ax{\textsf{Restaurant}(\textsf{gamberone})}, \\ &
	\ax{\textsf{Pub}(\textsf{irish})}, \quad \ax{\textsf{locIn}(\textsf{sioux},\textsf{modernArt})}, \\ &
	\ax[0.8]{\textsf{Popular}(\textsf{comic})}, \quad \ax[0.6]{\textsf{Cheap}(\textsf{irish})}, \quad
	\ax[0.7]{\textsf{near}(\textsf{irish},\textsf{comic})} & \}
\end{align*}
provides information about the specific attractions and services provided at the location. From this information, we can deduce,
for example, that the \textsf{modernArt} museum is a \textsf{TouristAttraction}, and is \textsf{Popular} to a degree at least
$0.6$.
Under the G\"odel t-norm, a possible model of \Oexa is depicted graphically in Figure~\ref{fig:model}, where any assertion
not depicted is considered to hold to degree 0.
\begin{figure}
\includegraphics[width=\textwidth]{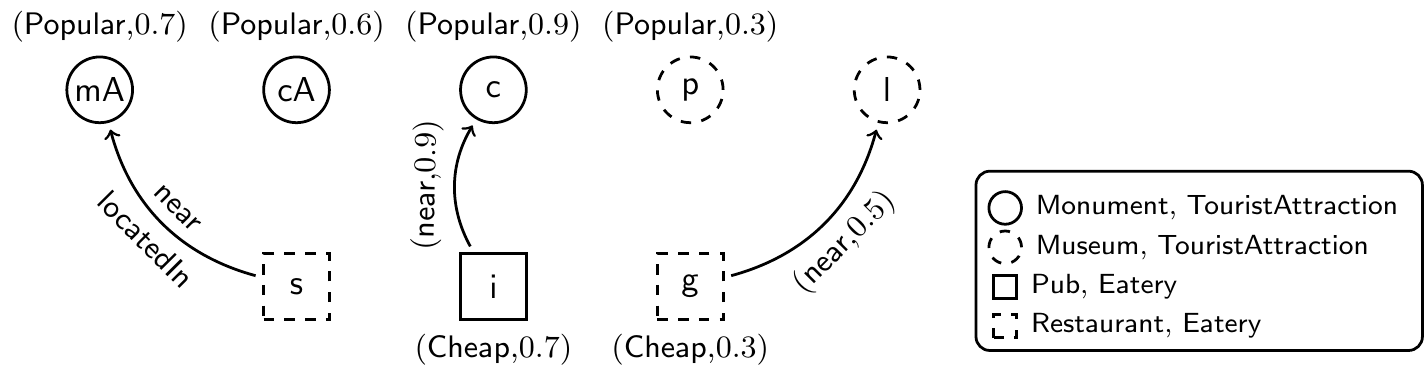}
\caption{A model for the ontology \Oexa from Example~\ref{exa:run}. Individual names are abbreviated to avoid cluttering, and
start with a lower-case letter as customary in DLs. The shape
and border of the nodes represent the crisp concepts, while vague concepts are associated to a degree. }
\label{fig:model}
\end{figure}
 For example, the model from Figure~\ref{fig:model} interprets the \textsf{irish}
pub as being \textsf{Cheap} to degree 0.7, which satisfies the constraint in the ABox requiring this degree to be at least 0.6. 
In addition, the \textsf{peace} monument is \textsf{Popular} to degree 0.3, even though there is no explicit requirement for
this in \Oexa.
Note that under this semantics, any model \Imc of \Oexa should necessarily satisfy that 
$\textsf{Cheap}^\Imc(\textsf{sioux}^\Imc)=0$; this is in fact the case in the model from Figure~\ref{fig:model}. 
Hence, adding any assertion of the form
\ax[d]{\textsf{Cheap}(\textsf{sioux})} with $d>0$ to this ontology would make it inconsistent.
\end{example}
For this paper, we are interested in answering two kinds of queries. The first kind are conjunctive queries, which consider whether 
a combination of facts can be derived from the knowledge in an ontology. In the fuzzy setting, the degree of such derivation
must also be taken into account. 

Let \NV be a set of \emph{variables}, which is disjoint from \NI, \NC, and \NR.
A \emph{term} is an element of $\NV\cup\NI$; that is, an individual name or a variable. An \emph{atom} is an expression of the
form $C(t)$ (concept atom) or $P(t_1,t_2)$ (role atom). Henceforth, \vx and \vy denote tuples of variables. 
\begin{definition}[conjunctive query]
A \emph{conjunctive query} (CQ) is a first-order formula of the form $\exists \vy.\phi(\vx,\vy)$ where $\phi$ is a conjunction
of atoms which only use the variables from \vx and \vy. 
The variables \vy are called \emph{existential variables}, and those in \vx are \emph{answer variables}. 
A \emph{union of conjunctive queries} (UCQ) is a finite set of CQs that use the same answer variables.
Henceforth, $\At(\phi)$ denotes the set of all atoms appearing in $\phi$.
\end{definition}
As in the classical setting, an answer to a conjunctive query, or a union of conjunctive queries, is only considered when
it is provided by every model of the ontology. This is usually known as a \emph{certain answer}.

Given the CQ $q(\vx)=\exists \vy.\phi(\vx,\vy)$, the interpretation \Imc, and a tuple of individuals \va of the same length
as \vx, a \emph{match} is a mapping $\pi$ which assigns to each $a\in\NI$ the value $a^\Imc$; to each variable in \vx the
corresponding element of $\va^\Imc$; and to each variable in \vy an element $\delta\in\Delta^\Imc$. We extend the match
$\pi$ to apply to assertions as follows: $\pi(B(t))=B(\pi(t))$ and $\pi(P(t_1,t_2))=P(\pi(t_1),\pi(t_2))$. The \emph{degree} of
the CQ $q(\vx)$ w.r.t.\ the match $\pi$ is 
\[
q^\Imc(\va^\Imc,\pi(\vy)):=\bigotimes_{\alpha\in\At(\phi)}(\pi(\alpha))^\Imc.
\]
That is, a match maps all the variables in the query to elements of the interpretation domain, where the tuple \va is
used to identify the mapping of the answer variables. The satisfaction or matching degree of the query is the (fuzzy) 
conjunction---that is, the t-norm---of the satisfaction or matching degrees of the atoms under this mapping. 
From now on, $\Pi(\Imc)$ denotes the set of all matches of $q(\vx)$ w.r.t.\ the interpretation \Imc. 
An important difference between classical query answering and our setting is that the 
fuzzy semantics provides a degree
to every possible atom. Hence, in reality $\Pi(\Imc)$ is always defined by the set of all tuples of individuals with length $|\vx|$. 
However, the degree of these matches varies and may often be zero.
For example, for the model \Imc in Figure~\ref{fig:model} and the query $q(x)=\textsf{Popular}(x)$, the set of all matches
$\Pi(\Imc)$ assigns to the variable $x$ any of the constants 
$\{\textsf{mA},\textsf{cA},\textsf{c},\textsf{p},\textsf{l},\textsf{s},\textsf{i},\textsf{g}\}$ to degrees 
$0.7,0.6,0.9,0.3,0,0,0$, and $0$, respectively.
When answering a query, one is often interested in the matches that hold to at least some degree $d$, as defined next.

\begin{definition}[degree queries]
%
%
A tuple of individuals \va is an \emph{answer} of the conjunctive query $q(\vx)$ to degree $d$ w.r.t.\ the interpretation 
\Imc (denoted by
$\Imc\models q(\va)\ge d$) iff $q^\Imc(\va^\Imc):=\sup_{\pi\in\Pi(\Imc)}q^\Imc(\va^\Imc,\pi(\vy))\ge d$. It is a
 \emph{certain answer} (or \emph{answer} for short) of $q(\vx)$ over the ontology \Omc to degree $d$ 
(denoted by $\Omc\models q(\va)\ge d$) iff $\Imc\models q(\va)\ge d$ holds for every model \Imc of \Omc.
%
The crisp set of certain answers of the query $q(\vx)$ w.r.t.\ \Omc and their degree is denoted by $\ans(q(\vx),\Omc)$; that is,
\[
\ans(q(\vx),\Omc):=\{(\va,d)\mid \Omc\models q(\va)\ge d \text{ and for all }d'>d, \Omc\not\models q(\va)\ge d'\}.
\]
\end{definition}
It is important to keep in mind that the atoms in a CQ are not graded, but simply try to match with elements in the domain as both concept and roles are interpreted as fuzzy relations (unary and binary, respectively).
The use of the truth degrees in the ontology becomes relevant in the degree of the answers found. 
Moreover, recall that every tuple of
individuals of length $|\vx|$ belongs to $\ans(q(\vx),\Omc)$, but with different associated degrees.

Returning to our example, while all individuals belong to the set $\ans(q(x))$, for the query $q(x)=\textsf{Popular}(x)$ 
to some degree, the certain answers for $q(x)$ w.r.t.\ \Oexa to degree at least 0.6 are only \textsf{modernArt}, 
\textsf{contArt}, and \textsf{comic}. The latter one is the only answer to degree at least 0.8.

The second kind of query we are interested in generalises that of degree queries, when considering the G\"odel semantics, by 
allowing a degree threshold for
each of the atoms in the conjunction, rather than for the overall conjunction. We formally define this class next.

\begin{definition}[threshold queries]
A \emph{threshold atom} is an expression of the form $\alpha\ge d$, where $\alpha$ is an atom and $d\in[0,1]$. 
A \emph{threshold query} (TQ) is a first-order formula of the form $\exists \vy.\phi(\vx,\vy)$ where $\phi$ is a 
conjunction of threshold atoms using only the variables from \vx and \vy.  
\end{definition}
The notion of a match and an answer to a threshold query are analogous to those of degree queries, with the proviso that 
the degree bounds apply at the level of atoms, and not at the level of queries.

\begin{definition}[TQ answer]
Given an interpretation \Imc and a tuple of individuals \va, the match $\pi$ \emph{satisfies} the threshold atom
$\alpha\ge d$ (denoted by $\pi\models\alpha\ge d$) iff $\alpha^\Imc\ge d$. It \emph{satisfies} the threshold query
$q(\vx)=\exists \vy.\phi(\vx,\vy)$ ($\pi\models q(\va)$) iff $\pi\models \alpha\ge d$ holds for every threshold atom in $q$.

A tuple of individuals \va is an \emph{answer} to the TQ $q(\vx)$ w.r.t.\ the interpretation \Imc ($\Imc\models q(\va)$)
iff there is a match $\pi$ w.r.t.\ \va and \Imc such that $\pi\models q(\va)$. It is a \emph{certain answer} of $q(\vx)$ over
the ontology \Omc iff for every model \Imc of \Omc it holds that $\Imc\models q(\va)$.
\end{definition}
Note that, differently from conjunctive queries, but in an analogous manner to degree queries, the answers to a threshold 
query are
\emph{not} graded. Indeed, a tuple \va may or may not be an answer, and we are interested in finding those tuples which
satisfy the degrees at each of the threshold atoms. 
In a sense, threshold queries provide a more fine-grained structure to deal with the properties of interest within a query in relation
to degree queries. 
Indeed, in a degree query, one can only provide an overall degree which should be obtained after the degrees of all the atoms
are conjoined via the t-norm. In particular, for non-idempotent t-norms and large queries, this conjunction will tend to be smaller
and smaller, and the degrees of the independent atoms have the same influence overall. Even when considering the idempotent 
G\"odel t-norm, a degree query $q(\vx)\ge d$ only expresses that all the atoms in $q$ should hold to degree at least $d$ 
(recall that the G\"odel t-norm refers to the minimum operator), but it is not possible to express that some atoms should hold
with a higher degree. A threshold query, on the other hand, is capable of requiring different degrees for each of the atoms.

\begin{example}
\label{exa:queries}
Suppose, in our running example, that we are interested in finding a cheap eatery that is nearby a popular tourist attraction,
and that we are using the G\"odel semantics.
This basic query could be expressed as%
\footnote{For brevity, we conjoin the atoms in a CQ through commas (`,') instead of $\land$.}
\[
q(x) = \exists y. \textsf{Cheap}(x), \textsf{Popular}(y), \textsf{near}(x,y).
\]
Since this query considers vague concepts and roles, we want to find answers that satisfy it to at least some degree. 
For the degree query $q(x)\ge 0.6$, the only possible answer is the \textsf{irish} pub. 

Suppose now that for us it is more important that the eatery is cheap than the popularity of the tourist attraction. For example,
even though we are content with the tourist attraction being popular to only degree 0.6, the eatery should be cheap to a degree
at least 0.8. This can be expressed through the threshold query
\[
q'(x) = \exists y. \textsf{Cheap}(x)\ge 0.8, \textsf{Popular}(y)\ge 0.6, \textsf{near}(x,y)\ge 0.6.
\]
In this case, the TQ has no answers w.r.t.\ the ontology \Omc. However, any answer to $q'$ would also be an answer to
$q(x)\ge 0.6$, as overall they define the same minimum over all the degrees of interest.
Note that this last claim only holds for the case of the G\"odel semantics. Indeed, as we will see later in this paper, for other
semantics degree queries are not properly special cases of TQs.
\end{example}
A class of conjunctive queries of special significance is that where the tuple of answer variables \vx is empty. This means that the 
answer tuple of individuals provided as an answer must also be empty. In the classical setting, these are called 
\emph{Boolean queries}, because they can only return a Boolean value: true if there is a match for the existential variables in 
every model, and false otherwise. In 
the fuzzy setting, the set of answers to such a query will only contain one element $((),d)$. Thus, in that case, we are only
interested in finding the degree $d$, and call those queries \emph{fuzzy queries}. This degree is the tightest value for which we 
can find a satisfying matching. Formally, the ontology \Omc \emph{entails} the 
fuzzy query $q()$ to degree $d$ iff $\Omc\models q()\ge d$ and $\Omc\not\models q()\ge d'$ for all $d'>d$.
Fuzzy queries allow us to find the degree of a specific answer \va without having to compute $\ans(q(\vx),\Omc)$: simply
compute the degree of the fuzzy query $q(\va)$.

In the case of threshold queries, we can also consider the special case where the answer tuple \vx is empty. In that case, 
as in the classical case, the only possible answer is the empty tuple (if there is a match which satisfies the query) or no answer
if no such match exists. For  that reason, in the case of threshold queries without answer variables we preserve the classical
terminology and call them Boolean (threshold) queries.

\medskip

As it is typically done for query answering in description logics, we consider two measures of complexity: 
\emph{data complexity}, where
only the size of the ABox (and the candidate answer, if any) is considered as part of the input, and \emph{combined complexity} in 
which the size of the whole ontology (including the TBox) is taken into account.%
\footnote{Note that our notion of \emph{combined complexity} does \emph{not} include the query as part of the input, but only the 
ontology. This view contrasts the usual database definition (and papers following it) where the combined complexity includes the
query, but is in line with the terminology used in ontology-based query answering; e.g.~\cite{ACKZ09}. The motivation is to understand
the influence of the knowledge in the complexity, abstracting from the query, which is already known to be a source of intractability
already for databases. In the context of ontology-based query answering, combined complexity is typically only used in
combination with simple fixed queries, which means that the query does not really have an important influence.} 
For data complexity, it is relevant to consider sub-linear complexity classes.
In particular, we consider \AC and \LS. For the full formal definitions, we refer the interested reader to \citeN{papa-complexity} and 
\citeN{BoSi90}. Here we only mention
briefly that evaluation of FO-queries over a database is in \AC on the size of the database \cite{Alice} and \AC is strictly contained
in \LS \cite{FuSS-MST84}. 

In classical \DLLR, query answering w.r.t.\ an ontology is reduced to the standard problem of query answering over a database
through a process known as query rewriting, and thus is in \AC w.r.t.\ data complexity. The main idea is to include in the query all 
the information that is required by the
TBox, in such a way that only assertions from the ABox need to be considered. 
In our running example, note that there is no assertion in the ABox \Aexa which explicitly mentions a tourist attraction. We only
know that the two monuments and the three museums are tourist attractions thanks to the TBox. In this case, the query
rewriting approach would take the query $q(x)=\textsf{TouristAttraction}(x)$ and transform it into the UCQ
\[
\{ \textsf{TouristAttraction}(x), \quad \textsf{Museum}(x), \quad \textsf{Monument}(x) \}
\]
looking ``backwards'' over the axioms in the TBox. The answers of this UCQ over the ABox alone are exactly those of the
original query over the whole ontology.

As seen in this simple example, there are many possible choices to create
the matches that comply with the TBox. Hence, this method results in a UCQ even if the original query is a simple CQ. 
At this point, the ABox is treated as a database, which 
suffices to find all the certain answers. Similarly, a special UCQ can be used to verify that the ontology is \emph{consistent};
that is, whether it is possible to build a model for this ontology. For the full details on how these query rewritings work in classical
\DLLR, see \cite{dl-lite}.
In terms of combined complexity, consistency can be decided in polynomial time; in fact, it is
\NLS-complete~\cite{ACKZ09}.

\section{The Canonical Interpretation}

A very useful tool for developing techniques for answering queries in \DLLR is the canonical interpretation. We first show 
that the same idea can be extended (with the necessary modifications) to fuzzy ontologies, independently of the t-norm 
underlying its semantics.

Let $\Omc=(\Tmc,\Amc)$ be a \DLLR ontology and assume w.l.o.g.\ that there are no axioms of the form 
$\left<\exists Q_1\sqsubseteq\exists Q_2,d\right>\in\Tmc$; any such axiom can be substituted by the two axioms
$\left<\exists Q_1\sqsubseteq A,1\right>,\left<A\sqsubseteq \exists Q_2,d\right>$ where $A$ is a new concept name not 
appearing in \Tmc.
The \emph{canonical interpretation} of \Omc is the interpretation $\Ican(\Omc)=(\Delta^\Ican,\cdot^\Ican)$ over the
domain $\Delta^\Ican:=\NI\cup\NN$---where \NN is a countable set of \emph{constants}---obtained through the following
(infinite) process. Starting from the \emph{empty} interpretation which sets $A^\Ican(\delta)=0$ and $P^\Ican(\delta,\eta)=0$ 
for every $A\in\NC, P\in\NR$ and $\delta,\eta\in\Delta^\Ican$, exhaustively apply the following rules:
\begin{enumerate}[label=\textbf{R\arabic*.}]
 \item\label{rule:r1} if $\left<A(a),d\right>\in\Amc$ and $A^\Ican(a)< d$, then update the value $A^\Ican(a):=d$;
 \item\label{rule:r2} if $\left<P(a,b),d\right>\in\Amc$ and $P^\Ican(a,b)< d$, then update the value $P^\Ican(a,b):=d$;
 \item\label{rule:r3} if $\left<A_1\sqsubseteq A_2,d\right>\in\Tmc$ and $A_2^\Ican(\delta)< A_1^\Ican(\delta)\otimes d$, then update
 	$A_2^\Ican(\delta):=A_1^\Ican(\delta)\otimes d$;
 \item\label{rule:r4} if $\left<A\sqsubseteq \exists P,d\right>\in\Tmc$ and for every $\eta\in\Delta^\Ican$,
 	$P^\Ican(\delta,\eta)<A^\Ican(\delta)\otimes d$ holds, then select a fresh element $\eta_0$ such that 
 	$P^\Ican(\delta,\eta_0)=0$ and
 	update the value $P^\Ican(\delta,\eta_0):=A^\Ican(\delta)\otimes d$;
 \item\label{rule:r5} if $\left<A\sqsubseteq \exists P^-,d\right>\in\Tmc$ and for every $\eta\in\Delta^\Ican$ 
 	$P^\Ican(\eta,\delta)<A^\Ican(\delta)\otimes d$ holds, then select a fresh element $\eta_0$ such that 
 	$P^\Ican(\eta_0,\delta)=0$ and
 	update the value $P^\Ican(\eta_0,\delta):=A^\Ican(\delta)\otimes d$;
 \item\label{rule:r6} if $\left<\exists P\sqsubseteq A,d\right>\in \Tmc$ and $\exists\eta\in\Delta^\Ican$ such that
 	$A^\Ican(\delta)<P^\Ican(\delta,\eta)\otimes d$, then update $A^\Ican(\delta):=P^\Ican(\delta,\eta)\otimes d$;
 \item\label{rule:r7} if $\left<\exists P^-\sqsubseteq A,d\right>\in \Tmc$ and $\exists\eta\in\Delta^\Ican$ such that
 	$A^\Ican(\delta)<P^\Ican(\eta,\delta)\otimes d$, then update $A^\Ican(\delta):=P^\Ican(\eta,\delta)\otimes d$;
 \item\label{rule:r8} if $\left<Q_1\sqsubseteq Q_2,d\right>\in\Tmc$ and $Q_2^\Ican(\delta,\eta)<Q_1^\Ican(\delta,\eta)\otimes d$, then
 	update $Q_2^\Ican(\delta,\eta)$ to the value $Q_1^\Ican(\delta,\eta)\otimes d$.
\end{enumerate}
where the rules are applied in a fair manner; that is, an applicable rule is eventually triggered. The process of
rule application is a monotone non-decreasing function, and as such has a least fixpoint, which is the canonical interpretation
$\Ican(\Omc)$.%
\footnote{By Tarski's Theorem \cite{Tars-55}, this fixpoint is the limit of the (fair) application of the rules starting from the smallest
element; in this case, the empty interpretation as described before.}

Intuitively, $\Ican(\Omc)$ should be a minimal model of \Omc, which describes the necessary conditions of all other models of \Omc.
Indeed, the first two rules ensure that the conditions imposed by the ABox are satisfied, by setting the degrees of the unary
and binary relations to the smallest required value. The remaining rules guarantee that 
all elements of the domain satisfy the positive axioms from the TBox, and each rule is as weak as possible in satisfying these
constraints.
The canonical interpretation of the ontology \Oexa from Example~\ref{exa:run} is depicted in Figure~\ref{fig:canonical}.
\begin{figure}
\includegraphics[width=\textwidth]{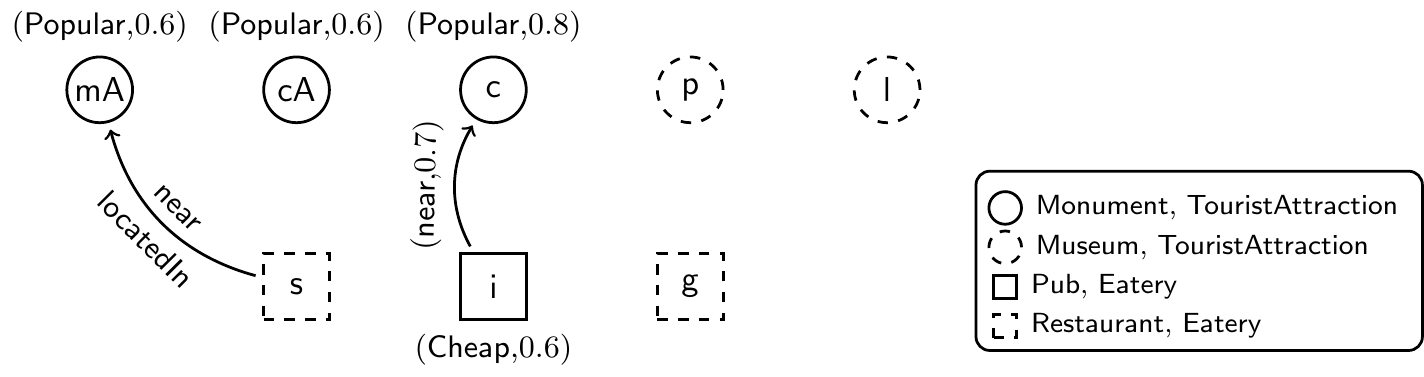}
\caption{The canonical interpretation for the ontology \Oexa from our running example.}
\label{fig:canonical}
\end{figure}
Note that in general it provides a lower membership degree of each individual to every concept when compared to the model
from Figure~\ref{fig:model}. This intuition justifies the name of \emph{canonical} interpretation.
As in the classical case, $\Ican(\Omc)$ can be homomorphically embedded in
every model of \Omc, and hence be used as a representative of them all. We show a similar result with the difference that in this 
case, the homomorphism needs to take into account 
the truth degrees from the interpretation function as well.%
\footnote{A careful reader will notice that the exhaustive application of the rules may produce different interpretations. We 
discuss this issue in further detail later in this section. For now, it suffices to know that all the possible canonical interpretations
are equivalent modulo homomorphisms.}
This is described in the following proposition.

\begin{proposition}
\label{prop:min:can}
Let \Omc be a consistent fuzzy \DLL ontology, 
$\Imc=(\Delta^\Imc,\cdot^\Imc)$ be a model of \Omc, and
$\Ican(\Omc)=(\Delta^\Ican,\cdot^\Ican)$ its canonical interpretation. 
There is a function $\psi$ from 
$\Delta^\Ican$ to $\Delta^\Imc$ such that:
\begin{enumerate}
 \item for each $A\in\NC$ and $\delta\in\Delta^\Ican$, $A^\Ican(\delta)\le A^\Imc(\psi(\delta))$; and
 \item for each $P\in\NR$ and $\delta,\eta\in\Delta^\Ican$, $P^\Ican(\delta,\eta)\le P^\Imc(\psi(\delta),\psi(\eta))$.
\end{enumerate} 
\end{proposition}
\begin{proof}
Let $\Omc=(\Tmc,\Amc)$.
We construct the function $\psi$ recursively through the rule applications that define $A^\Ican$, and show that the two properties
from the proposition are invariant w.r.t.\ the rule applications. We define first $\psi(a)=a^\Imc$ for all $a\in \NI$. Recall that 
initially, $A^\Ican(\delta)=P^\Ican(\delta,\eta)=0$ for all $A\in \NC, P\in\NR, \delta,\eta\in\Delta^\Ican$. Hence, the properties
hold trivially in this case.

Assume now that the properties hold before a rule application; we show that they also hold afterwards by a case analysis
over the rule used:
\begin{enumerate}[label=\textbf{R\arabic*.}]
\item if $\left<A(a),d\right>\in\Amc$,  since \Imc is a model of this axiom, it follows that $A^\Imc(a^\Imc)\ge d$. The
rule application sets $A^\Ican(a)=d$ and hence $A^\Ican(a)\le A^\Imc(\psi(a))=A^\Imc(a^\Imc)$.
\item if $\left<P(a,b),d\right>\in\Amc$, the rule application sets $P^\Ican(a,b)=d$. Since \Imc satisfies this axiom,
it follows that $P^\Imc(\psi(a),\psi(b))=P^\Imc(a^\Imc,b^\Imc)\ge d=P^\Ican(a,b)$.
\item if $\left<A_1\sqsubseteq A_2,d\right>\in\Tmc$, the rule application over a given $\delta\in\Delta^\Ican$ updates
$A_2^\Ican(\delta)$ to $A_1^\Ican(\delta)\otimes d$. Since \Imc satisfies this axiom, by the induction hypothesis and monotonicity
of $\otimes$ we know that
$A_2^\Imc(\psi(\delta))\ge A_1^\Imc(\psi(\delta))\otimes d\ge A_1^\Ican(\delta)\otimes d =A_2^\Ican(\delta)$.
\item if $\left<A\sqsubseteq \exists P,d\right>\in\Tmc$, let $\delta$ be the element over which the rule is applicable, and $\eta_0$
the fresh element selected by the rule application. Since \Imc is a model, we know that there exists an element $\kappa\in\Delta^\Imc$
such that $P^\Imc(\psi(\delta),\kappa)\ge A^\Imc(\psi(\delta))\otimes d$. We thus define $\psi(\eta_0):=\kappa$. By the induction
hypothesis and monotonicity of $\otimes$ we get 
$P^\Imc(\psi(\delta),\psi(\eta_0))P^\Imc(\psi(\delta),\kappa)\ge A^\Imc(\psi(\delta))\otimes d\ge A^\Ican(\psi(\delta))\otimes d=
	P^\Ican(\delta,\eta_0)$.
\item if $\left<A\sqsubseteq \exists P^-,d\right>\in\Tmc$, let $\delta$ be the element over which the rule is applicable, and $\eta_0$
the fresh element selected by the rule application. Since \Imc is a model, we know that there exists an element $\kappa\in\Delta^\Imc$
such that $P^\Imc(\kappa,\psi(\delta))\ge A^\Imc(\psi(\delta))\otimes d$. We thus define $\psi(\eta_0):=\kappa$. By the induction
hypothesis and monotonicity of $\otimes$ we get 
$P^\Imc(\psi(\eta_0),\psi(\delta))P^\Imc(\kappa,\psi(\delta))\ge A^\Imc(\psi(\delta))\otimes d\ge A^\Ican(\psi(\delta))\otimes d=
	P^\Ican(\eta_0,\delta)$.
\item if $\left<\exists P\sqsubseteq A,d\right>\in\Tmc$, then for the chosen $\delta,\eta\in\Delta^\Ican$ we have by the induction
hypothesis that $A^\Ican(\delta)=P^\Ican(\delta,\eta)\otimes d\le P^\Imc(\psi(\delta),\psi(\eta))\otimes d \le A^\Imc(\psi(\delta))$.
\item if $\left<\exists P\sqsubseteq A,d\right>\in\Tmc$, then for the chosen $\delta,\eta\in\Delta^\Ican$ we have by the induction
hypothesis that $A^\Ican(\delta)=P^\Ican(\eta,\delta)\otimes d\le P^\Imc(\psi(\eta),\psi(\delta))\otimes d \le A^\Imc(\psi(\delta))$.
\item if $\left<Q_1\sqsubseteq Q_2,d\right>\in\Tmc$, the rule application over given $\delta,\eta\in\Delta^\Ican$ updates
$Q_2^\Ican(\delta,\eta)$ to $Q_1^\Ican(\delta,\eta)\otimes d$. Since \Imc satisfies this axiom, by the induction hypothesis and 
monotonicity of $\otimes$ we know that
$$
Q_2^\Imc(\psi(\delta),\psi(\eta))\ge Q_1^\Imc(\psi(\delta),\psi(\eta))\otimes d\ge Q_1^\Ican(\delta,\eta)\otimes d =Q_2^\Ican(\delta,\eta).
$$
\end{enumerate}
Hence, the result holds after the fair application of all possible rules.
\end{proof}
Importantly, note that the construction of $\Ican(\Omc)$ does not take the negations into account; e.g., the axiom
\ax[0.5]{\exists\textsf{locIn}\sqsubseteq \neg\textsf{Cheap}} is never used during this construction. 
The effect of this is that 
$\Ican(\Omc)$ might not be a model of \Omc at all. 

\begin{example}
\label{exa:incons}
Consider the fuzzy \DLLR ontology $\Oexa=(\Texo,\Aexo)$ where 
\begin{align*}
\Texo:={} &\{\left<A_1\sqsubseteq\neg A_2,1\right>\}, \\
\Aexo:={} &\{\left<A_1(a),0.5\right>,\left<A_2(a),0.5\right>\}.
\end{align*}
Under the G\"odel semantics, by application of the first rule, the canonical interpretation maps $A_1^\Ican(a)=A_2^\Ican(a)=0.5$. 
However, this violates
the axiom in \Texo, which requires that $A_1^\Ican(a)\Rightarrow\ominus A_2^\Ican(a)=1$. That is, it requires that
$A_1^\Ican(a)<\ominus A_2^\Ican(a)$, which is only possible when $A_1^\Ican(a)=0$ or $A_2^\Ican(a)=0$.
Note that a similar phenomenon could be observed also in the TBox \Texa of our running example, which contains an
axiom with a negated concept.
\end{example}
The issue is that the negative axioms may introduce inconsistencies, by enforcing upper bounds in the degrees used, which
are not verified by the canonical interpretation; recall, in fact, that the previously described construction monotonically increases
the degrees to satisfy the minimal requirements, but never verifies whether these degrees affect some upper bound. 
On the other hand, we can prove that, as long as there is a model, $\Ican(\Omc)$ is one.

\begin{proposition}
\label{prop:can:model}
$\Ican(\Omc)$ is a model of \Omc iff \Omc is consistent. 
\end{proposition}
\begin{proof}
The \emph{only if} direction is trivial, hence we focus on showing that if \Omc is consistent, then $\Ican(\Omc)$ is a model of \Omc.
Note first that, by construction, $\Ican(\Omc)$ satisfies all positive axioms. Otherwise, a rule would trigger, and since the construction
applies all rules fairly until exhaustion, no rule is applicable in the resulting interpreation $\Ican(\Omc)$. Hence, if $\Ican(\Omc)$
is not a model of \Omc, there must exist a negative axiom of the form (i) $\left<B\sqsubseteq \neg C,d\right>$ or
(ii) $\left<Q\sqsubseteq \neg R,d\right>$ that is not satisfied by the canonical interpretation. We consider the case (i); the other
case can be treated analogously.

If $\Ican(\Omc)\not\models\left<B\sqsubseteq \neg C,d\right>$, then there must exist an element $\delta\in\Delta^\Ican$ such that
$B^\Ican(\delta)\Rightarrow (\neg C)^\Ican(\delta)<d$ or, equivalently, $\ominus C^\Ican(\delta)< B^\Ican(\delta)\otimes d$.
Since \Omc is consistent, there must exist a model $\Imc=(\Delta^\Imc,\cdot^\Imc)$ of \Omc. By Proposition \ref{prop:min:can}, there 
exists a function $\psi:\Delta^\Ican\to \Delta^\Imc$ such that, in particular, $B^\Ican(\delta)<B^\Imc(\psi(\delta))$ and
$C^\Ican(\delta)<C^\Imc(\psi(\delta))$. By antitonicity of $\ominus$, the latter means that 
$\ominus C^\Imc(\psi(\delta))\le \ominus C^\Ican(\delta)$ and hence
\[
\ominus C^\Imc(\psi(\delta))\le \ominus C^\Ican(\delta) < B^\Ican(\delta)\otimes d \le B^\Imc(\psi(\delta))\otimes d
\]
But this means that $\Imc\not\models \left<B\sqsubseteq \neg C,d\right>$, which contradicts the assumption that \Imc was a model
of \Omc.
\end{proof}
It can be seen that the ontology \Oexo from Example \ref{exa:incons} is inconsistent under the G\"odel semantics. On the 
other hand, under 
the \L ukasiewicz semantics, \Oexo is in fact consistent which, by this proposition, means that $\Ican(\Omc)$ is a model of this
ontology. This is easily confirmed by recalling that the \L ukasiewicz negation is involutive; that is $\ominus d=1-d$. In the
case of the example, we have $\ominus 0.5=0.5$; the axiom \ax{A_1\sqsubseteq \neg A_2} is satisfied because
$0.5 = A_1^\Ican(a) \le (\neg A_2)^\Ican(a)=0.5$.

The consequence of the last two propositions is that $\Ican(\Omc)$ is complete for existential positive queries, and in particular for 
conjunctive queries and threshold queries.
\begin{corollary}
\label{cor:ican}
If \Omc is a consistent fuzzy \DLLR ontology, then 
\begin{enumerate}
 \item for every CQ $q(\vx)$, answer tuple \va, and $d\in[0,1]$ it holds that $\Omc\models q(\va)\ge d$ iff 
 $\Ican(\Omc)\models q(\va)\ge d$;
 \item for every TQ $q(\vx)$ and answer tuple \va, $\Omc\models q(\va)$ iff $\Ican(\Omc)\models q(\va)$.
\end{enumerate}
\end{corollary}
\begin{proof}
Proposition \ref{prop:can:model} states that $\Ican(\Omc)$ is a model; hence anything that does not follow from it cannot be
an answer. On the other hand, Proposition \ref{prop:min:can} states that the degree of every atom in any model is at least the
degree given by \Ican, and hence if a tuple is an answer in \Ican, it is also an answer in every other model.
\end{proof}

\subsection*{A short note on the canonical interpretation}

Before delving deeper into the process of answering queries (the main contribution of this paper), it is worth considering the
canonical interpretation in more detail, starting with the definite article used in its naming. Indeed, although we always
speak about \emph{the} canonical interpretation, the actual structure produced is not necessarily unique, and depends on the
order in which rules are chosen to apply, specially in relation to rules \textbf{R4} and \textbf{R5}, which introduce new relevant
elements. This is highlighted in the following example.

\begin{example}
\label{exa:can:mult}
Consider an ontology containing the axioms
\begin{align*}
\Amc := {} & \{ \left<A(a), 1\right>, \left<B(a), 1\right> \} \\
\Tmc := {} & \{ \left<A\sqsubseteq \exists R, 0.3 \right>, \left<B\sqsubseteq \exists R, 0.5 \right> \}
\end{align*}
After applying the rule \textbf{R1} over the ABox axioms, we have an interpretation where $A^\Ican(a)=B^\Ican(a)=1$. 
At this point, rule \textbf{R4} is applicable for any of the two axioms in \Tmc. If we first apply it to the first axiom, we select a 
fresh element, e.g.\ $\eta_0$ and set $R(a,\eta_0)=0.3$; at this point, the rule is still applicable to the second axiom. This application
requires selecting a new fresh element (now $\eta_1$) and set $R(a,\eta_1)=0.5$. At this point, no rules are applicable and we have
a canonical interpretation. 

If instead we chose first to apply the rule on the second axiom, we would choose a fresh element (say, $\eta_2$) and set
$R(a,\eta_2)=0.5$. This application immediately disallows the application of the rule to the first axiom, and hence the process
stops.
\end{example}
Note that the two interpretations built in this example are not equivalent (see Figure \ref{fig:can:cons}).
\begin{figure}[tb]
\includegraphics[width=\textwidth]{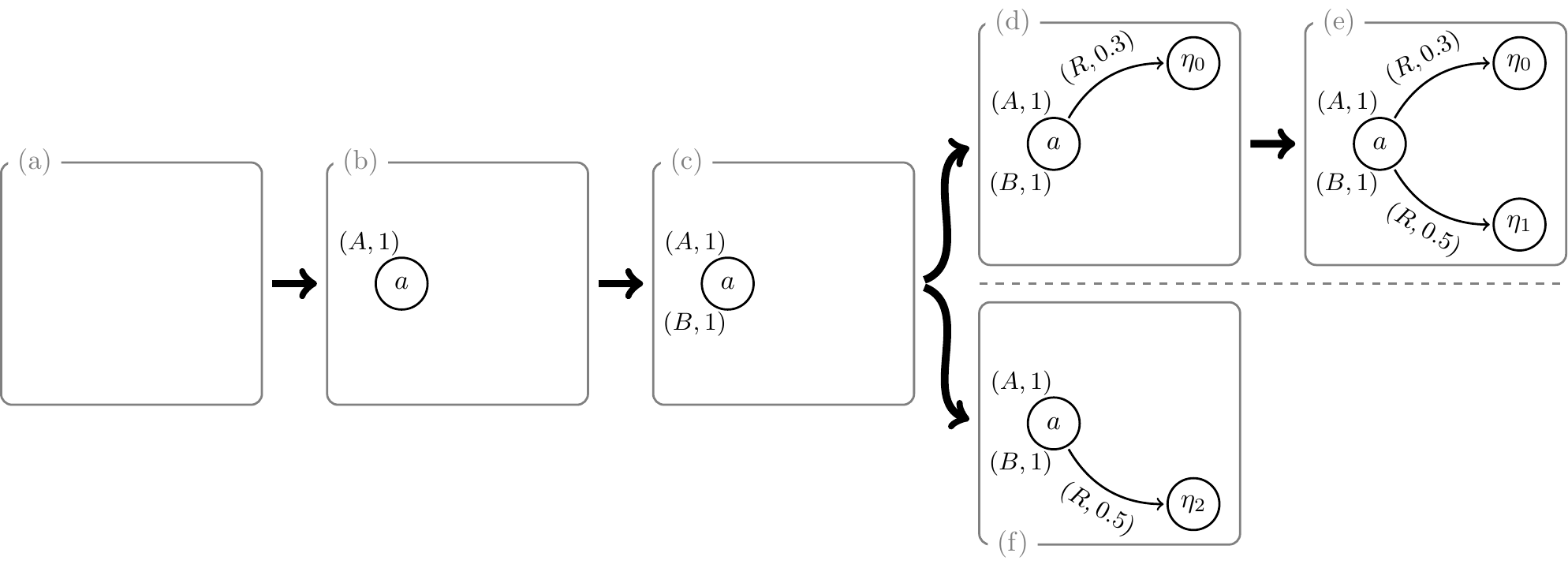}
\caption{Two canonical interpretation constructions from the ontology in Example \ref{exa:can:mult}. From the empty interpretation
(a), \textbf{R1} is applied to each assertion to reach (c). One can either apply \textbf{R4} to $\left<A\sqsubseteq \exists R, 0.3 \right>$
and go through the upper branch (d) to build the interpretation (e); or to $\left<B\sqsubseteq \exists R, 0.5 \right>$ and obtain (f) 
directly.}
\label{fig:can:cons}
\end{figure}
However, they are homomorphic in the sense specified by Proposition \ref{prop:min:can}. This is not a coincidence. In fact, note that
the proofs of Propositions \ref{prop:min:can} and \ref{prop:can:model} do not depend on the order of rule applications, but only on
the fact that these rules were exhaustively (and fairly) applied. If the ontology is consistent, by the latter proposition the interpretation
obtained is a model, regardless of the order chosen, and by the former proposition, it is homomorphic to all the interpretations which
can be derived following different application orderings. In other words, the canonical interpretation is unique \emph{up to
homomorphism}. In the following, we disregard this issue and consider an arbitrary, but fixed, canonical interpretation as unique.

\bigskip

We now return to the issue of answering queries. Corollary \ref{cor:ican} states that these queries can be answered through
the canonical interpretation.
Obviously, such an approach is impractical; in fact, impossible, because it is an infinite model 
constructed through an infinitary
process. Additionally, we still have the burden to prove that the ontology is consistent, which is a prerequisite for the use of 
Corollary \ref{cor:ican} to answer queries. Fortunately, for the G\"odel and product t-norms, we can resort to existing 
results from the literature for this latter task.

\begin{definition}[classical version]
Let $\Omc=(\Tmc,\Amc)$ be a fuzzy \DLL ontology. The \emph{classical version} $\widehat\Omc$  of \Omc is defined by
$\widehat\Omc:=(\widehat\Tmc,\widehat\Amc)$, where
\begin{align*}
\widehat\Tmc:={} & \{ B\sqsubseteq C \mid \left<B\sqsubseteq C,d\right>\in\Tmc, d>0\} \cup
			 \{ Q\sqsubseteq R \mid \left<Q\sqsubseteq R,d\right>\in\Tmc, d>0\}, \\
\widehat\Amc:={} & \{ B(a) \mid \left<B(a),d\right>\in\Tmc, d>0\} \cup
			 \{ P(a,b) \mid \left<P(a,b),d\right>\in\Tmc, d>0\}.
\end{align*}
\end{definition}
That is, $\widehat\Omc$ contains all the axioms and assertions from \Omc which hold with a positive degree---note that
any fuzzy axiom or assertion with degree 0 could be removed w.l.o.g.\ anyway. 
The following result is a direct consequence of work on more expressive fuzzy DLs \cite{BoDP-AIJ15}.
\begin{proposition}
\label{prop:reduc}
Let \Omc be a G-\DLLR or $\Pi$-\DLLR ontology. Then \Omc is consistent iff $\widehat\Omc$ is consistent.
\end{proposition}
In those cases, consistency checking can be reduced to the classical case, without the need to modify the query or the basic 
formulation of the ontology.
For the ontology \Oexo in Example \ref{exa:incons}, we have 
$\widehat\Oexo=(\{A_1\sqsubseteq \neg A_2\},\{A_1(a),A_2(a)\})$, which is inconsistent in the classical case, thus showing
(through Proposition~\ref{prop:reduc}) that it is inconsistent under the G\"odel and product t-norm semantics.
We note that the example also shows that Proposition \ref{prop:reduc} does not hold for the \L ukasiewicz t-norm, since
we have established that \Oexo is consistent under this semantics, although its classical version remains inconsistent under
classical interpretations.

A particular consequence of Proposition \ref{prop:reduc} is that deciding consistency of G-\DLLR and $\Pi$-\DLLR ontologies is in
\AC w.r.t.\ data complexity, and \NLS-complete w.r.t.\ combined complexity, where the \NLS lower bound comes from known
results in classical \DLLR \cite{ACKZ09}. Thus adding truth degrees does not affect the 
complexity of this basic reasoning task.
We now turn our attention to the task of query answering with the different semantics, starting with the idempotent case
of the G\"odel t-norm. We consider first the case of conjunctive queries, which allows for a simple solution, and then study
threshold queries for which a rewriting technique is needed.

Before studying how to answer queries over fuzzy \DLLR ontologies and its complexity, we note that in the case that an
ontology is classical---i.e., it uses only degree 1 in all its axioms---its canonical interpretation constructed as described in
this section is equivalent to the classical canonical interpretation from \cite{dl-lite}. This fact will be used in the following sections.

\section{Answering Conjunctive Queries over G\"odel Ontologies}

For this and the following section, we are always considering the G\"odel t-norm as the underlying operator for interpreting
all fuzzy statements, and in particular the conjunctive queries.

The G\"odel semantics are very limited in their expressivity. On the one hand, we have seen that $\ominus d\in\{0,1\}$ for all
$d\in[0,1]$. This means that whenever we have an axiom of the form $\left<B\sqsubseteq \neg B',d\right>$ or 
$\left<Q\sqsubseteq \neg Q',d\right>$ with $d>0$, we are in fact saying that for every element $\delta\in\Delta^\Imc$,
if $B^\Imc(\delta)>0$, then $B'^\Imc(\delta)=0$---because in this case $\ominus B'^\Imc(\delta)=1$, which is the only possible
way of satisfying the axiom. A similar argument holds for role axioms.
Thus, for this section we can assume w.l.o.g.\ that all negative axioms hold with degree 1; i.e., they are of the form
\ax{B\sqsubseteq \neg B'} or \ax{Q\sqsubseteq \neg Q'}.
On the other hand, a positive axiom of the form $\left<B\sqsubseteq B',d\right>$ requires that for every $\delta\in\Delta^\Imc$,
$B'^\Imc(\delta)\ge \min\{B^\Imc(\delta),d\}$. That is, the only way to guarantee that an atom gets a high degree is to
use axioms with a high degree. We use these facts to reduce reasoning tasks in this setting to the classical \DLLR scenario.

Consider a consistent G-\DLLR ontology \Omc. We can decide a lower bound for the 
degree of a CQ simply by querying a \emph{cut} of \Omc. 
\begin{definition}[cut ontology]
Given a value $\theta\in(0,1]$, the 
\emph{$\theta$-cut} of the ontology \Omc is defined as the sub-ontology $\Omc_{\ge \theta}:=(\Tmc_{\ge \theta},\Amc_{\ge \theta})$ 
where
\begin{align*}
\Tmc_{\ge \theta}:={} & \{ \left<\gamma,e\right>\in\Tmc \mid e\ge \theta\}, \\
\Amc_{\ge \theta}:={} & \{ \left<\alpha,e\right>\in\Amc \mid e\ge \theta\}.
\end{align*}
\end{definition}
That is, $\Omc_{\ge \theta}$ is the subontology containing only the axioms and assertions that hold to degree at least $\theta$.
To show that $\theta$-cuts suffice for answering queries, we use the canonical interpretation.

Note that including new axioms or assertions to an ontology would result in an
update of the canonical interpretation which only increases the degree of some of the elements of the domain. More precisely,
if $\Ican(\Omc)$ is the canonical interpretation of $\Omc=(\Tmc,\Amc)$, then the canonical interpretation of 
$\Omc'=(\Tmc\cup\{\left<B\sqsubseteq C,d\right>\},\Amc)$ is the result of applying the construction rules starting from 
$\Ican(\Omc)$. This holds because the resulting canonical interpretation is not dependent on the order in which rules are
applied (and hence axioms taken into account) as long as this is done fairly.%
\footnote{Formally, different rule application orderings yield homomorphic interpretations. See the discussion in the previous section.}
Since
$\Ican(\Omc)$ has already applied all the rules on axioms of \Omc exhaustively, the only remaining rule applications will be based on
the new axiom $\left<B\sqsubseteq C,d\right>$ and new applications over \Tmc arising from it. Under the G\"odel semantics, all 
the updates
increase the interpretation function up to the value $d$; that is, if $\cdot^\Icanp$ is the interpretation function of
$\Ican(\Omc')$, the difference between $\Ican(\Omc)$ and $\Ican(\Omc')$ is that there exist some elements such that 
$A^\Ican(\delta)<A^\Icanp(\delta)=d$, and similarly for roles there exist some pairs $\delta,\eta$ such that 
$P^\Ican(\delta,\eta)<P^\Icanp(\delta,\eta)=d$. For all others, the degrees remain unchanged.
Moreover, if $d_0$ is the smallest degree appearing in the ontology \Omc, then its canonical interpretation uses only truth
degrees in $\{0\}\cup[d_0,1]$; that is, no truth degree in $(0,d_0)$ appears in $\Ican(\Omc)$.
With these insights we are ready to produce our first results. Recall, once again, that for the rest of this section, we always 
consider that the semantics is based on the G\"odel t-norm; i.e., we have a G-\DLLR ontology.

\begin{lemma}
\label{lem:cut}
Let \Omc be a consistent G-\DLLR ontology, $q(\vx)$ a query, \va a tuple of individuals, and $\theta\in(0,1]$. Then
$\Omc\models q(\va)\ge \theta$ iff $\Omc_{\ge \theta}\models q(\va)\ge \theta$.
\end{lemma}
\begin{proof} 
Since $\Omc_{\ge \theta}\subseteq \Omc$, every model of \Omc is also a model of $\Omc_{\ge \theta}$. Hence, if
$\Omc_{\ge \theta}\models q(\va)\ge \theta$, then $\Omc\models q(\va)\ge \theta$.

For the converse, assume that $\Omc_{\ge \theta}\not\models q(\va)\ge \theta$. By Corollary \ref{cor:ican}, this means that
$\Ican(\Omc_{\ge \theta})\not\models q(\va)\ge \theta$. That is, $q^\Ican(\va^\Ican)< \theta$. Let 
$\Ican(\Omc)=(\Delta^\Icanp,\cdot^\Icanp)$ be the canonical interpretation of \Omc. Recall that the difference between
\Omc and $\Omc_{\ge \theta}$ is that the former has some additional axioms with degrees smaller than $\theta$. As argued before,
this means that the difference between $\Ican(\Omc)$ and $\Ican(\Omc_{\ge \theta})$ are just some degrees, which are all
smaller than $\theta$; that is, for every $A\in\NC$, $P\in\NR$, and $\delta,\eta\in\Delta^\Icanp$, if $A^\Icanp(\delta)\ge \theta$,
then $A^\Ican(\delta)\ge \theta$ and if $P^\Icanp(\delta,\eta)\ge \theta$, then $P^\Ican(\delta,\eta)\ge \theta$. By assumption, this
means that $q^\Icanp(\va^\Icanp)<\theta$, and hence $\Ican(\Omc)\not\models q(\va)\ge \theta$. Thus, 
$\Omc\not\models q(\va)\ge \theta$.
\end{proof}
What this lemma states is that in order to find a lower bound for the degree of a query, one can ignore all the axioms and
assertions that provide a smaller degree than the bound we are interested in. However, one still needs to answer a query for 
a fuzzy ontology ($\Omc_{\ge \theta}$ is still fuzzy), for which we still do not have any effective method. The following lemma 
solves this issue, considering the classical version of this ontology.
\begin{lemma}
\label{lem:class}
Let \Omc be a consistent G-\DLLR ontology such that $\Omc_{\ge \theta}=\Omc$ for some $\theta>0$. Then, 
$\Omc\models q(\va)\ge \theta$ iff
$\widehat\Omc\models q(\va)$.
\end{lemma}
\begin{proof}
Every model of $\widehat\Omc$ is also a model of \Omc, with the additional property that the interpretation function maps 
all elements to $\{0,1\}$. If $\Omc\models q(\va)\ge \theta>0$, then for every model \Imc of $\widehat\Omc$ it holds that
$q^\Imc(\va^\Imc)\ge \theta>0$, and thus $q^\Imc(\va^\Imc)=1$, which means that $\widehat\Omc\models q(\va)$.

Conversely, if $\widehat\Omc\models q(\va)$, the canonical interpretation $\Ican(\Omc)$ must be such that 
$q^\Ican(\va^\Ican)>0$; but as argued before, since \Omc only has axioms and assertions with degrees $\ge \theta$, it must
be the case that all degrees of $\Ican(\Omc)$ are in $\{0\}\cup[\theta,1]$, and hence $q^\Ican(\va^\Ican)\ge \theta$. This implies,
by Corollary \ref{cor:ican} that $\Omc\models q(\va)\ge \theta$.
\end{proof}
Note that the condition of this lemma, which requires that $\Omc_{\ge \theta}=\Omc$, is only stating that all the degrees in the 
ontology \Omc are
at least $\theta$. That condition is immediately satisfied by a cut ontology, and hence the lemma can be applied directly to it.

Lemmas~\ref{lem:cut} and~\ref{lem:class} together provide a method for reducing answering degree queries over 
G\mbox{-}\DLLR ontologies to query answering in classical \DLLR.

\begin{theorem}
\label{thm:reduct}
If \Omc is a consistent G-\DLLR ontology and $\theta>0$, then it holds that $\Omc\models q(\va)\ge \theta$ iff 
$\widehat\Omc_{\ge \theta}\models q(\va)$.
\end{theorem}
This means that we can use a standard ontology-based query answering system to answer fuzzy queries in \DLLR as well.
Note that the approach proposed by Theorem \ref{thm:reduct} can only decide whether the degree of an answer to a query
is at least $\theta$, but it needs the value $\theta\in (0,1]$ as a parameter. If, instead, we are interested in computing the degree of an
answer, or $\ans(q(\vx),\Omc)$, we can still use a classical query answering method as an underlying black-box aid as 
described next.

Since the TBox \Tmc and the ABox \Amc which compose the ontology \Omc are both finite, the set 
$\Dmc:=\{d\mid \left<\alpha,d\right>\in\Tmc\cup\Amc\}$ of degrees appearing in the ontology is also finite; in fact, its size
is bounded by the size of \Omc. Hence, we can assume that \Dmc is of the form $\Dmc=\{d_0,d_1,\ldots,d_n,d_{n+1}\}$
where $d_0\ge 0,d_{n+1}=1$ and for all $i,0\le i\le n$, $d_{i}<d_{i+1}$. In order to find the degree of an answer \va to a query
$q$, we proceed as follows: starting from $i:=n+1$, we iteratively ask the query $\Omc_{\ge d_i}\models q(\va)$ and
decrease $i$ until the query is answered affirmatively, or $i$ becomes 0 (see Algorithm \ref{alg:degree}). 
\begin{algorithm}[tb]
\DontPrintSemicolon
\KwData{Ontology \Omc, query $q$, answer \va, $\Dmc=\{d_0,d_1,\ldots,d_{n+1}\}$}
\KwResult{The degree of $q(\va)$ w.r.t.\ \Omc}
$i\gets n+1$ \;
$\Nmc\gets\widehat\Omc_{\ge 1}$ \;
\While{$\Nmc\not\models q(\va)$ \textbf{and} $i>0$ }{
  $i \gets i-1$ \;
  $\Nmc\gets\widehat\Omc_{\ge d_i}$ \;
}
\Return $d_i$ \;
\caption{Compute the degree of an answer to a query}
\label{alg:degree}
\end{algorithm}
In the former case, $d_i$ is the degree for $q(\va)$; in the latter, the degree is 0---i.e., \va is not an answer of $q$.%
\footnote{Note that the algorithm can be made more efficient using a binary search, instead of a linear decrease of available
degrees. We chose this presentation to provide a clear association with Corollary \ref{cor:logspace}.}

During the execution of this algorithm, each classical query needed at line 3 can be executed in \AC (and in particular in \LS) 
in the size of the data; i.e., the ABox as shown in \cite{ACKZ09}. The iterations in the loop do not affect the overall space used, 
as one 
can simply produce a new query every time and clean up the previous information. Overall, this means that the degree of an 
answer can be computed in \LS in data complexity, using a classical query answering engine. 

\begin{corollary}
The degree of an answer \va to a query $q$ w.r.t.\ the G-\DLLR ontology \Omc is computable in logarithmic space w.r.t.\ the size
of the ABox (i.e., in data complexity).
\end{corollary}
We will later see that this upper bound can indeed be reduced to \AC by seeing a degree query as a special case of a threshold
query. However, the method that provides a tight complexity bound requires a new implementation of the rewriting approach,
with all its associated optimizations, in contrast to the method from Algorithm~\ref{alg:degree}, which can simply call any
existing classical tool; e.g. \cite{GLL+ORE12,CCG+SEBD15}.

Computing the whole set of pairs $\ans(q(\vx),\Omc)$ is a more complex task. Although we can follow an approach similar to 
Algorithm \ref{alg:degree}, where the answers to $q(\vx)$ are computed for each ontology $\widehat\Omc_{\ge d_i}$, 
in order to assign the appropriate degree to each answer, we need to either keep track of all the answers found so far, or add
a negated query which excludes the answers with a higher degree. In both cases, we require a different approach and
a potential larger use of memory.
On the other hand, the whole set of answers $\ans(q(\vx),\Omc)$ will usually contain many answers that hold with a very low
degree, which may not be of much interest to the user making the query. When dealing with degrees, a more meaningful
task is to find the $k$ answers with the highest degree, for some natural number $k$; i.e., the \emph{top-$k$ answers} of
$q$.

Algorithm \ref{alg:degree} once again suggests a way to compute the top-$k$ answers. As in the algorithm, one starts with
the highest possible degree, and expands the classical ontology by including the axioms and assertions with a lower degree. The
difference is that one stops now when the query returns at least $k$ tuples as answers. At that point, the tuples found are
those with the highest degree for the query. As before, each of these queries can be answered in \AC in data complexity,
which yields a \LS upper bound for answering top-$k$ queries in data complexity. 

\begin{corollary}
\label{cor:logspace}
Top-$k$ queries over consistent G-\DLLR ontologies can be answered in logarithmic space w.r.t.\ the size of the ABox. 
\end{corollary}

\section{Threshold Queries over G\"odel Semantics}
\label{sec:tq}

We now turn our attention to the case of threshold queries, but keeping the assumption of the G\"odel semantics in place.
The first thing to notice when considering threshold queries is that the simple approach developed in the previous section,
where one calls a classical query answering engine over a cut subontology, cannot work. Indeed, as each atom needs to
be satisfied potentially to a different degree, there is no one cut that can suffice to answer them all. Indeed, we have already
seen a TQ in Example~\ref{exa:queries} which has no answers even though the natural cut ontology provides one answer.

When considering Boolean threshold queries, it may be tempting to simply try to verify each threshold atom separatedly through
a cut ontology. However, such an approach is not sound due to the existentially quantified variables which need to be 
associated to a (potentially anonymous) individual. This problem is not new, as it arises already for conjunctive queries over
classical databases.

To answer these queries, we will adapt the query rewriting technique from the classical setting. The underlying idea is 
essentially the same, as described previously in this paper, where an atom $B(x)$ may be substituted by an atom $C(x)$ if
the TBox contains the axiom $C\sqsubseteq B$. However, one has to be careful with the degrees used. In fact,
some axioms may not be applied during the rewriting, if their degree is not large enough.

\begin{example}
Consider once again the TBox \Texa from Example~\ref{exa:run}, and suppose that we are interested in finding all popular
attractions, up to a given degree $d\in[0,1]$; that is, we have the query $q(x)=\textsf{Popular}(x)\ge d$. The TBox contains 
the axiom \ax[0.6]{\textsf{Museum}\sqsubseteq \textsf{Popular}}. This means that answers to $q(x)$ w.r.t.\ this TBox should 
also include the answers to $\textsf{Museum}(x)$, but this depends on the value of $d$, as we explain next.

Suppose that $d>0.6$; e.g., if we have $q(x)=\textsf{Popular}(x)\ge 0.7$. For an individual $a$ and any model \Imc
of the ontology, we have no guarantee that $\textsf{Popular}^\Imc(a^\Imc)\ge 0.7$ regardless of the degree of
$\textsf{Museum}^\Imc(a^\Imc)$. Indeed, even if $\textsf{Museum}^\Imc(a^\Imc)=1$, the only thing that can be guaranteed
is that the degree of $a$ belonging to \textsf{Popular} is at least $0.6$, which does not suffice to become a positive answer
to the query. Hence, there is no need to include \textsf{Museum} in the rewriting.

Suppose now that $d\le 0.6$; e.g., with the query $q(x)=\textsf{Popular}(x)\ge 0.5$.
In this case, we note that every individual $a$ such that $\textsf{Museum}^\Imc(a^\Imc)\ge 0.5$ must satisfy also that
$\textsf{Popular}^\Imc(a^\Imc)\ge 0.5$. Indeed, recall that under the G\"odel semantics, $f\Rightarrow e$ is either $e$ if
$e\le f$ or $1$ otherwise. Since \Imc satisfies the axiom \ax[0.6]{\textsf{Museum}\sqsubseteq \textsf{Popular}},
whenever $f=\textsf{Museum}^\Imc(a^\Imc)\ge 0.5$ holds, we know that the degree $e$ of $\textsf{Popular}^\Imc(a^\Imc)$
must be such that $f\Rightarrow e\ge 0.6$. If $e\le f$, this can only be true if $e\ge 0.6>0.5$. Otherwise, we know that 
$e>f\ge 0.5$, and hence any individual belonging to the concept \textsf{Museum} to degree at least $0.5$ is an answer to
the query $q(x)$.
\end{example}
This example shows that during the rewriting process, we only need to consider the axioms that hold to a degree greater
than the threshold of the current atom of interest. During the rewriting step, the original threshold is preserved regardless
of the bound from the axioms. We now proceed to describe the rewriting process in detail, following the ideas developed
originally for classical \DLLR and other members of the \DLL family through the \textsf{PerfectRef} algorithm. 
To aid understanding from readers knowledgeable with
the original method, we preserve as much of the terminology from~\cite{dl-lite} as possible.
From now on, in a query $q(\vx)=\varphi(\vx,\vy)$, we call all the variables in \vx \emph{distinguished}, and any variable that
appears at least twice within a query \emph{shared}. Note that there is no need to keep track of the variables that are not
distinguished nor shared (from now on, called undistinguished, unshared variables); it is only relevant that they can be adequately 
assigned a value. Hence, those variables will be
denoted by an underscore (`\uv'), and use $y=\uv$ to express that $y$ is one such variable.

\begin{definition}[applicability]
An axiom $\alpha$ is \emph{applicable} to the threshold atom $A(x)\ge d$ iff $\alpha$ is of the form \ax[e]{C\sqsubseteq A} and
$d\le e$. It is \emph{applicable} to the threshold atom $P(x_1,x_2)$ iff either (i) $x_2=\uv$ and $\alpha$ is of the form
\ax[e]{C\sqsubseteq \exists P} with $d\le e$; (ii) $x_1=\uv$ and $\alpha$ is of the form \ax[e]{C\sqsubseteq\exists P^-};
or (iii) $\alpha$ is of the form \ax[e]{Q\sqsubseteq P} or \ax[e]{Q\sqsubseteq P^-} with $d\le e$.

If $\alpha$ is applicable to the threshold atom $\gamma$, the \emph{result} of the application is the atom $gr(\gamma,\alpha)$
defined through the rules in Figure~\ref{fig:rules}.
\begin{figure}
\begin{itemize}
\item If $\gamma=A(x)\ge d$ and $\alpha=\ax[e]{A_1\sqsubseteq A}$, then $gr(\gamma,\alpha)=A_1(x)\ge d$
\item If $\gamma=A(x)\ge d$ and $\alpha=\ax[e]{\exists P\sqsubseteq A}$, then $gr(\gamma,\alpha)=P(x,\uv)\ge d$
\item If $\gamma=A(x)\ge d$ and $\alpha=\ax[e]{\exists P^-\sqsubseteq A}$, then $gr(\gamma,\alpha)=P(\uv,x)\ge d$
\item If $\gamma=P(x,\uv)\ge d$ and $\alpha=\ax[e]{A\sqsubseteq \exists P}$, then $gr(\gamma,\alpha)=A(x)\ge d$
\item If $\gamma=P(x,\uv)\ge d$ and $\alpha=\ax[e]{\exists P_1\sqsubseteq \exists P}$, then $gr(\gamma,\alpha)=P_1(x,\uv)\ge d$
\item If $\gamma=P(x,\uv)\ge d$ and $\alpha=\ax[e]{\exists P_1^-\sqsubseteq \exists P}$, then $gr(\gamma,\alpha)=P_1(\uv,x)\ge d$
\item If $\gamma=P(\uv,x)\ge d$ and $\alpha=\ax[e]{A\sqsubseteq \exists P^-}$, then $gr(\gamma,\alpha)=A(x)\ge d$
\item If $\gamma=P(\uv,x)\ge d$ and $\alpha=\ax[e]{\exists P_1\sqsubseteq \exists P^-}$, then $gr(\gamma,\alpha)=P_1(x,\uv)\ge d$
\item If $\gamma=P(\uv,x)\ge d$ and $\alpha=\ax[e]{\exists P_1^-\sqsubseteq \exists P^-}$, then $gr(\gamma,\alpha)=P_1(\uv,x)\ge d$
\item If $\gamma=P(x_1,x_2)\ge d$ and $\alpha\in\{\ax[e]{P_1\sqsubseteq P},\ax[e]{P_1^-\sqsubseteq P^-}\}$ then
	$gr(\gamma,\alpha)=P_1(x_1,x_2)\ge d$
\item If $\gamma=P(x_1,x_2)\ge d$ and $\alpha\in\{\ax[e]{P_1\sqsubseteq P^-},\ax[e]{P_1^-\sqsubseteq P}\}$ then
	$gr(\gamma,\alpha)=P_1(x_2,x_1)\ge d$
\end{itemize}
\caption{The result $gr(\gamma,\alpha)$ of applying the axiom $\alpha$ to the threshold atom $\gamma$.}
\label{fig:rules}
\end{figure}
\end{definition}
The \textsf{PerfectRef} algorithm constructs a union of threshold queries by iteratively substituting atoms $\gamma$ for which 
an axiom $\alpha$
is applicable, with the result $gr(\gamma,\alpha)$ of the application. This follows the idea of tracing backwards the axioms
in order to absorb the TBox into the query which was previously outlined. The pseudocode for \textsf{PerfectRef} is more 
formally described in Algorithm~\ref{alg:pr}. In the algorithm, $q[\gamma,\eta]$ is the query resulting from
substituting in $q$ the atom $\gamma$ with the atom $\eta$. The function $reduce(p,\gamma_1,\gamma_2)$ called in 
line \ref{alg:pr:red} simply returns the query obtained by applying the most general unifier between $\gamma_1$ and $\gamma_2$
to $p$. For unification, all nondistinguished, unshared variables are considered different. For simplicity, we always assume that
all nondistinguished, unshared variables are known, and hence call them \uv when testing applicability.
\begin{algorithm}[tb]
\DontPrintSemicolon
\KwData{Threshold query $q$, G-\DLLR TBox \Tmc}
\KwResult{Union of threshold queries $T$}
$T\gets \{q\}$ \;
\Repeat{$T'=T$}{
	$T' \gets T$ \;
	\For{\textbf{each} $p\in T'$}{
		\For{\textbf{each} $\gamma\in p$, and \textbf{each} $\alpha\in\Tmc$}{
			\If{$\alpha$ is applicable to $\gamma$}{
				$T \gets T\cup \{p[\gamma/gr(\gamma,\alpha)]\}$\;
			}
		}
		\For{\textbf{each} $\gamma_1,\gamma_2\in p$}{
			\If{$\gamma_1$ and $\gamma_2$ unify}{
				$T\gets T\cup\{reduce(p,\gamma_1,\gamma_2)\}$ \; \label{alg:pr:red}
			}
		}
	}
}
\Return $T$ \;
\caption{\textsf{PerfectRef}}
\label{alg:pr}
\end{algorithm}

Note that, just as in the classical case, the application of the $reduce$ function is necessary to guarantee correctness of the
rewriting. Specifically, a variable that is bound in a query $p$ may become unbound after the unification process, which may
allow more axioms to be applied for the rewriting.

Once again, the algorithm takes as input a threshold query $q$, and returns a union of threshold queries $T$ which is 
constructed by taking into account the information from the TBox \Tmc. The importance of this rewriting is that at this point, 
the answers to the original query $q$ w.r.t.\ an ontology $\Omc=(\Tmc,\Amc)$ can be obtained by applying the query $T$
to the ABox \Amc, seen as a standard database. 

Let $db(\Amc)$ be the ABox \Amc seen as a database. Note that since we have fuzzy assertions, the database will contain
binary relations (representing concept assertions) and ternary relations (representing the role assertions), where the last 
element of the relation is the degree; a number in the interval $[0,1]$. Under this view, a threshold query can also be seen
as a conjunctive query, taking into account the inequalities in the selection. Given a union of threshold queries $T$, $UCQ(T)$
denotes the fact that $T$ is being read as a UCQ in this sense. Given an ABox \Amc and a union of TQs $T$, we denote by
$\ans(db(\Amc),UCQ(T))$ the set of answers to $T$ w.r.t.\ $db(\Amc)$ from a database perspective. We also denote
by $\ans(q,\Omc)$ the set of answers to the TQ $q$ w.r.t.\ the ontology \Omc. We then obtain the following result.

\begin{theorem}
\label{thm:rewriting}
Let $\Omc=(\Tmc,\Amc)$ be a consistent G-\DLLR ontology, $q$ a TQ, and $T$ the union of TQs obtained through the rewriting. 
Then $\ans(q,\Omc)=\ans(db(\Amc),UCQ(T))$.
\end{theorem}
A consequence of Theorem~\ref{thm:rewriting} is that, in terms of data complexity, answering a TQ w.r.t.\ a \DLLR ontology
is at most as costly as answering a CQ over a database. Indeed, note that althought the query $q$ is transformed into a 
larger UCQ, the data itself remains unchanged. This yields the following result.

\begin{theorem}
Answering threshold queries w.r.t.\ consistent G-\DLLR ontologies is in \AC w.r.t.\ data complexity.
\end{theorem}
Before finishing this section, we return to a question on complexity left open in the previous section; namely, the precise
complexity of finding the degree of an answer to a conjunctive query. To answer this question, we first note that under the
G\"odel semantics, we can always see a degree query as a special case of a threshold query.

Given a CQ $q$, let $\At(q)$ be the set of all the atoms in $q$. For a degree $d\in[0,1]$, we can define the TQ
$TQ(q,d)=\bigwedge_{\gamma\in\At(q)}\gamma\ge d$. That is, $TQ(q,d)$ uses the same atoms as $q$, but assigns a minimum
degree of $d$ to each of them. Since the G\"odel semantics interprets the conjunction through the minimum operator, any 
answer of $TQ(q)$ yields a degree of at least $d$ to the original query $q$. 

\begin{lemma}
\label{lem:cdtotq}
Let \Omc be a consistent G-\DLLR ontology, $q$ a CQ, \va an answer tuple, and $d\in[0,1]$. It holds that $\Omc\models q(\va)\ge d$
iff $\Omc\models TQ(q(\va),d)$.
\end{lemma}
In order to find the degree of an answer, we can simply add as an answer variable after the rewriting one that looks at the
degrees from the database $db(\Amc)$. This does not affect the overall data complexity, and hence remains in \AC.

\begin{corollary}
Answering conjunctive queries w.r.t.\ consistent G-\DLLR ontologies is in \AC in data complexity.
\end{corollary}
This finishes our analysis of the G\"odel t-norm, which also provides our main results. In the following section we briefly
visit the case where the underlying t-norm is not idempotent, and showcase that in general dealing with such semantics
becomes harder.

\section{Non-idempotent t-norms}

We now move our attention to the t-norms that are not idempotent; in particular the product and \L ukasiewicz t-norms. 
Unfortunately, as we will see, the correctness of the
reductions and algorithms presented in the previous sections rely strongly on the idempotency of the G\"odel t-norm, and 
does not transfer
directly to the other cases. However, at least for the product t-norm, it is still possible to answer some kinds of 
queries efficiently. 

First recall that Proposition \ref{prop:reduc} holds for the product t-norm as well. Hence, deciding consistency of 
a $\Pi$-\DLLR ontology remains reducible to the classical case and thus, efficient. We now show with simple examples that the 
other results do not transfer so easily.

\begin{example}
\label{exa:prod}
Consider the ontology $\Oexb:=(\Texb,\Aexb)$ where
$\Texb:=\{\left<A_i\sqsubseteq A_{i+1},0.9\right>\mid 0\le i <n\}$ and
$\Aexb:=\{\left<A_0(a),1\right>\}$. Note that $\Oexb=({\Oexb})_{\ge 0.9}$, but the degree for the query $q()=A_n(a)$
is $0.9^n$ which can be made arbitrarily small by making $n$ large.
\end{example}
Similarly, it is not possible to find the top-$k$ answers simply by layering the $\theta$-cuts for decreasing values of $\theta$ until
enough answers can be found.
\begin{example}
Let $\Oexb':=(\Texb,\Aexb')$, where $\Aexb':=\Aexb\cup\{\left<A_n(b),0.85\right>\}$ and \Texb, \Aexb are as in 
Example \ref{exa:prod}. The top answer for $q(x)=A_n(x)$ is $b$ with degree 0.85, but from $({\Oexb'})_{\ge 0.9}$
we already find the answer $a$, which is not the top one.
\end{example}
The main point with these examples is that, from the lack of idempotency of the t-norm $\otimes$, we can obtain low degrees
in a match which arises from combining several axioms and assertions having a high degree. 
On the other hand, the product behaves well for positive values in the sense that applying the t-norm to two positive values
always results in a positive value; formally, if $d,e>0$, then $d\otimes e>0$. Thus, if we are only interested in knowing
whether the result of a query is positive or not, there is no difference between the G\"odel t-norm and the product t-norm.

\begin{definition}
A tuple \va is a \emph{positive answer} to the query $q(\vx)$ w.r.t.\ the ontology \Omc (denoted by $\Omc\models q(\va)>0$) iff 
for every model \Imc of \Omc it holds that $q^\Imc(\va^\Imc)>0$.
\end{definition}
\begin{theorem}
If \Omc is a consistent $\Pi$-\DLLR ontology, then $\Omc\models q(\va)>0$ iff $\widehat\Omc\models q(\va)$.
\end{theorem}
\begin{proof}
Every model of $\widehat\Omc$ is also a model of \Omc, with the additional property that the interpretation function maps 
all elements to $\{0,1\}$. If $\Omc\models q(\va)>0$, then for every model \Imc of $\widehat\Omc$ it holds that
$q^\Imc(\va^\Imc)>0$ and thus $q^\Imc(\va^\Imc)=1$, which means that $\widehat\Omc\models q(\va)$.

Conversely, if $\widehat\Omc\models q(\va)$, then the canonical interpretation is such that $q^\Ican(\va^\Ican)>0$, and
hence for every model \Imc it also holds that $q^\Imc(\va^\Imc)>0$.
\end{proof}
This means that, for the sake of answering positive queries over the product t-norm, one can simply ignore all the truth
degrees and answer a classical query using any state-of-the-art engine. In particular, this means that positive answers 
can be found in \AC in data complexity 
just as in the classical case.

\medskip

We now briefly consider the \L ukasiewicz t-norm, which is known to be the hardest to handle due to its involutive negation
and nilpotence, despite being in many cases the most natural choice for fuzzy semantics \cite{BoCP-JAIR17}.
As mentioned already, Proposition \ref{prop:reduc} does not apply to the \L ukasiewicz t-norm. That is, 
there are consistent \L\mbox{-}\DLLR ontologies whose classical version is inconsistent (see
Example \ref{exa:incons}). As a result, there is currently no known method for deciding consistency of these ontologies, let 
alone answering queries. The culprits for this are the involutive negation, which is weaker than the negation used in the
other two t-norms, but also the nilpotence, which may combine positive degrees to produce a degree of 0. The latter also
means that, even if one could check consistency, it is still not clear how to answer even positive queries.

\begin{example}
Consider the ontology $\Oexc:=(\Texc,\Aexc)$ where 
\begin{align*}
\Texc:={} &\{\left<A_0\sqsubseteq A_1,0.5\right>,\left<A_1\sqsubseteq A_2,0.5\right>\} \\ 
\Aexc:={} & \{\left<A_0(a),1\right>\}.
\end{align*}
Note that \Oexc is consistent, but there is a model \Imc (e.g., the canonical interpretation) of this ontology which sets
$A_2^\Imc(a^\Imc)=0$. Hence, $a$ is not a positive answer to the query $q(x)=A_2(x)$ even though
it is an answer of $q(x)$ over $\widehat\Oexc$.
\end{example}
Importantly, if we extend \DLLR with the possibility of using conjunctions as constructors for complex concepts, 
one can show following the ideas from \cite{BoCP-JAIR17,BoCP-PRUV14} that
deciding consistency of a \L-\DLLR ontology is \NP-hard in combined complexity even if negations are disallowed;
see Appendix \ref{app:NP} for full details.
In the classical case, this logic---which is called \DLLH---has a polynomial time consistency problem \cite{ACKZ09}.
This gives an indication that dealing with \L-\DLLR may also lead to an increase in complexity.

\medskip

Interestingly, the rewriting technique from Section~\ref{sec:tq} also works for other 
t-norms---modulo some basic modifications---when answering threshold queries. Recall, for example, that given an axiom
\ax[e]{A\sqsubseteq B}, and a threshold atom $B(x)\ge d$, if $e\ge d$ then the rewriting technique would substitute this
atom with $A(x)\ge d$. Although this substitution is sound for the idempotent G\"odel t-norm, it does not work directly for
the other ones. For example, under the product t-norm, if we set $d=e=0.9$ we note that guaranteeing $A^\Imc(x)\ge 0.9$ 
does not necessarily implies, in a model \Imc of \ax[e]{A\sqsubseteq B} that $B^\Imc(x)\ge 0.9$. Indeed, as long as 
$B^\Imc(x)\ge 0.81$, the axiom is satisfied in this case. A similar argument can be made for the \L ukasiewicz t-norm. 
Hence, we need to increase the required degree for the rewritten atom. 

Recall from the properties of the residuum that for every t-norm $\otimes$ it holds that 
$A^\Imc(x)\Rightarrow B^\Imc(x)\ge e$ iff $B^\Imc(x)\ge A^\Imc(x)\otimes e$. Thus, to ensure that $B^\Imc(x)\ge d$
it suffices to guarantee that $A^\Imc(x)\otimes e\ge d$. In the case of the product t-norm, this is akin to the condition
$A^\Imc(x)\ge d/e$. For the \L ukasiewicz t-norm the condition translates to the inequality $A^\Imc(x)\ge \min\{1, d+1-e\}$. We can then
apply the same \textsf{PerfectRef} algorithm, with a new definition of the function $gr$ that changes the last degree
(which is always $\ge d$ in Figure~\ref{fig:rules}) with the new degree developed here. Overall, this yields the following complexity
result.

\begin{theorem}
Answering threshold queries w.r.t.\ consistent \DLL ontologies is in \AC in data complexity.
\end{theorem}
Note that this theorem does not solve the problems sketched before for non-idempotent t-norms. Indeed, it is still not
clear how to check for consistency of a \L-\DLL ontology. Moreover, this result cannot be used to answer CQs because
the analogous to Lemma~\ref{lem:cdtotq} does not hold. Indeed, suppose that we have a simple CQ with only two atoms:
\[
q(x)= A(x) \land B(x).
\]
To turn $q(x)\ge d$ into a TQ, we need to assign a threshold to each of the atoms. Note however that, under a non-idempotent
t-norm, we cannot assign the same degree $d$ to each atom, as their conjunction would become in fact lower than $d$. 
To be more precise consider the product t-norm and $d=0.9$. Note that an answer to the TQ 
$A(x)\ge 0.9 \land B(x)\ge 0.9$ is \emph{not} necessarily an answer to $q(x)\ge 0.9$ because there could be a model 
that assigns both atoms to degree 0.9; hence the product of those degrees is $0.81<0.9$. To use a TQ, we need to choose
two degrees $d_1,d_2$ such that $d_1\cdot d_2=0.9$, and construct the TQ $A(x)\ge d_1 \land B(x)\ge d_2$. But there
are infinitely many choices to make in this regard, hence we cannot even construct a finite UTQ. Thus, unfortunately, although
we are able to answer TQs efficiently (if the ontology is known to be consistent), degree queries remain an open problem for
non-idempotent t-norms. 

\section{Conclusions}

In this paper we have studied the problem of answering queries over fuzzy ontologies written in \DLL. Our goal was to 
cover the gap in this area left by previous research. Indeed, although query answering w.r.t.\ ontologies is still an active topic, 
most work referring to fuzzy terminologies or ABoxes focused on the so-called Zadeh semantics, which does not preserve
desired properties from the mathematical fuzzy logic point of view. To our knowledge, only Mailis and Turhan 
\cite{MaTu-JIST14,MaTZ-DL15} have studied this problem based on t-norms, and found solutions based on the G\"odel t-norm.
However, they limited their approach to \emph{classical} TBoxes. They left open the problems of dealing with graded
TBoxes, handling threshold queries, and dealing with non-idempotent t-norms. 

A second goal of our work was to reuse as much as possible the classical techniques, in order to avoid an implementation
overhead when our algorithms will be, in future work, implemented and tested. As a result, we developed a method for 
answering degree queries which relies heavily on a classical query answering tool as a black box. Through this method, we
can take advantage of all the existing optimisations and improvements from that area, and simply use a better tool whenever
it becomes available without having to worry about the internal intricacies that make it work. In few words, our algorithm
for answering CQs w.r.t.\ the G\"odel semantics simply considers the classical version of the cut of the ontology. That is,
the method ignores all axioms that hold to a degree lower than the threshold imposed, and then sees the remaining query
answering question as a classical one. We emphasise that this approach works perfectly even if the TBox is graded. This means
that our results improve those from \cite{MaTu-JIST14} by allowing for fuzzy TBox axioms and not requiring a new rewriting
of the query.

Dealing with threshold queries, where each atom can be independently assigned a different degree, turns out to be more 
complex technically. In fact, we were not able to produce a direct reduction to a classical query answering problem---and it
is unlikely that such a reduction could exist, given the nature of the graded axioms. However, we could still exploit the main
ideas from the classical scenario, adapting the well-known \textsf{PerfectRef} method to the fuzzy scenario. In some sense 
\textsf{PerfectRef} absorbs the TBox into the query, forming a larger UCQ which can be answered classically, seeing the 
ABox as a database. In our case, we also need to take care of the degree at which the rewritten atoms should hold, when
creating the new query. Under the G\"odel semantics, it suffices to preserve the same degree from the original query, but for
non-idempotent t-norms the degree has to be increased accordingly to avoid including spurious answers. Importantly,
this shows that answering threshold queries w.r.t.\ consistent fuzzy ontologies is in the same complexity class (\AC) in data 
complexity as for the classical case, regardless of the t-norm underlying the semantics. The only caveat is that it is not 
known how to verify consistency of a fuzzy \DLL ontology in general. The idempotency of the G\"odel t-norm allowed us then
to show that CQ answering w.r.t.\ consistent G-\DLL ontologies is also in \AC in data complexity. This latter bound does not
hold for non-idempotent t-norms.

It is worth noting that the methods for answering degree and threshold queries both ultimately rely on a rewriting of the query 
to represent the information expressed in the TBox. While the rewriting does not affect the \emph{data} complexity, it is well
known that the UCQ obtained through \textsf{PerfectRef} may grow exponentially \cite{dl-lite,PMH10:elrewritingtodatalog}. This means
that, depending on the instance, answering these queries may still be impractical. For that reason, different rewriting and
answering techniques have been developed and tested; for example, rewriting into a Datalog program instead of an UCQ
\cite{GoSc-KR12,GoOP-ICDE11}. Our approach for solving threshold queries, given its reliance on \textsf{PerfectRef}, suffers
from the same drawbacks. In order to use more optimised approaches, it is necessary to study whether other existing rewritings
can also be adapted to the fuzzy setting. On the other hand, the approach to degree queries is, as mentioned already, fully
black box: we only need to call an unmodified classical query answering tool repeatedly. This allows us to directly \emph{plug}
whichever system performs best, without worrying about the implementation overhead of understanding and adapting the
existing tools. 

Through illustrative examples, we showed that dealing with CQs is in general harder when the underlying t-norm is not 
idempotent. The main issue is that there is no unique way to decide the bounds for the different degrees to which atoms
should hold to satisfy the lower bound for a conjunction of atoms. The problem is exacerbated by the fact that it is not even
clear how to decide consistency of ontologies under the nilpotent t-norm. Indeed, even in the absence of an ABox, the
\L ukasiewicz t-norm may impose upper bounds in the degrees of some assertions which are not obvious to detect, and could
contradict other conditions.

As
future work, we are also interested in implementing and testing our ideas, with the help of some fuzzy ontologies which
will be developed for specific application domains.

 \bibliographystyle{acmtrans}
 \bibliography{main_tplpv3X}

\appendix

\section{NP-Hardness for the Horn Case}
\label{app:NP}

In this appendix we show that ontology consistency in \L-\DLLH, a DL closely related to \L-DLLR, is \NP-hard. The proof builds
on the idea originally developed for a finitely-valued logic \cite{BoCP-PRUV14},%
\footnote{In reality, we use a simplified form based on 3-SAT, rather than the more general m-SAT from that work.}
extended with the methods from
\cite{BoCP-JAIR17,BoCP-IJCAI15} to deal with infinitely-valued t-norms. For brevity, we consider only a restricted version of
\L-\DLLH which already suffices to show hardness.

A \L-\DLLH GCI is an expression of the form $\left<B\sqsubseteq C,d\right>$ where $B,C$ are built through the grammar
$B::= A\mid B\sqcap B\mid \bot$, where $A\in \NC$; that is, they are conjunctions of concept names, and $d\in[0,1]$. The notion of
an ABox, a TBox, and an ontology are analogous to the \DLLR setting. The semantics of this logic is defined as for \DLLR, with
the addition that the interpretation function is defined for conjunctions as 
$(B\sqcap C)^\Imc(\delta):=B^\Imc(\delta)\otimes C^\Imc(\delta)$ and for bottom as $\bot^\Imc(\delta):=0$.

We show \NP-hardness through a reduction from the well-known problem of satisfiability of 3-CNF formulas \cite{Cook-SAT71}.
Very briefly, a 3-clause is a disjunction of exactly three literals (variables or negated variables) and a 3-CNF formula is a 
conjunction of 3-clauses. 

The main idea behind the reduction is to use an intermediate degree greater than 0 to simulate the logical truth value \emph{false}
from classical propositional logic. This allows us to simulate the disjunction from the 3-clauses through a conjunction of concepts.
This idea is formalised in full detail next.

Consider the set \Vmc of propositional variables. For each $v\in\Vmc$ we define two concept names $A_v$ and $A'_v$ with the
intuitive meaning that $A_v$ stands for the valuation making $v$ true, and $A'_v$ for the valuation making $v$ false.
We define the function $\rho$ which maps propositional literals to concept names as: 
\[
\rho(\ell) := 
	\begin{cases}
	 A_v & \text{if } \ell=v\in\Vmc \\
	 A'_v & \text{if } \ell=\neg v, v\in\Vmc.
	\end{cases}
\]
This function is extended to 3-clauses by defining $\rho(\ell_1\lor \ell_2\lor \ell_3):=\rho(\ell_1)\sqcap \rho(\ell_2)\sqcap \rho(\ell_3)$.
Abusing the notation, we identify the 3-CNF formula $\varphi$ with the set of 3-clauses it contains.

Let $\varphi$ be a propositional formula in 3-CNF, and $\var(\varphi)$ represent the set of all propositional variables appearing in
$\varphi$. We construct the ontology $\Omc_\varphi$ consisting of the very simple ABox 
$\Amc_\varphi:=\{\left<A_0(a),1\right>\}$
and the TBox
\begin{align}
\Tmc_\varphi := \{ & 
	\left<A_v \sqcap A_v \sqcap A_v \sqsubseteq A_v \sqcap A_v \sqcap A_v \sqcap A_v,1\right>, \label{tbox:one}\\ & 
	\left<A'_v \sqcap A'_v \sqcap A'_v \sqsubseteq A'_v \sqcap A'_v \sqcap A'_v \sqcap A'_v,1\right>,  \label{tbox:two}\\ &
	\left<A_v \sqcap A'_v \sqsubseteq \bot,1/3\right>, \left<A_0\sqsubseteq A_v\sqcap A'_v, 2/3\right>\mid v\in\var(\varphi)\} \cup {} \label{tbox:three} \\ 
	\{ &
	\left<A_0 \to \rho(c),1/3\right> \mid c\in\varphi\} \label{tbox:four}
\end{align}

\begin{theorem}
The formula $\varphi$ is satisfiable iff the ontology $\Omc_\varphi=(\Amc_\varphi,\Tmc_\varphi)$ is consistent.
\end{theorem}
\begin{proof}
We start with some observations about the TBox $\Tmc_\varphi$. The axioms in lines \eqref{tbox:one} and \eqref{tbox:two} require
that every model $\Imc=(\Delta^\Imc,\cdot^\Imc)$ of $\Tmc_\varphi$ is such that $A_v^\Imc(\delta)\in [0,2/3]\cup\{1\}$ and
${A'_v}^\Imc(\delta)\in [0,2/3]\cup\{1\}$ for all $\delta\in\Delta^\Imc$.%
\footnote{Intuitively, if the interpretation is between 2/3 and 1, then conjoining three times will necessarily yield a greater value than
conjoining four times, due to monotonicity. The only way to avoid this is to make the conjunction reach 0, or stay in 1.}
Given the ABox axiom, the GCIs from line \eqref{tbox:three} guarantee that \emph{exactly one} between $A_v^\Imc(a^\Imc)$
and ${A'_v}^\Imc(a^\Imc)$ is interpreted as $1$, and the other one as $2/3$. Intuitively, the value $2/3$ will be read as ``false'' and
$1$ as ``true.''
Finally, the axioms in line \eqref{tbox:four} guarantee that for every clause $c\in\varphi$, at least one of the conjuncts in 
$\rho(c)$ is interpreted as $1$ (i.e., ``true'') at the element $a^\Imc$. Now we prove the property.

If $\varphi$ is satisfiable, let \Vmc be a valuation making $\varphi$ true. As customary, a valuation is a subset of $\var(\varphi)$ that
expresses which variables are mapped to true. We construct the interpretation $\Imc_\Vmc=(\{a\},\cdot^\Imc)$
having a singleton domain, where $A_0^\Imc(a)=1$ and for each $v\in\var(\varphi)$
\begin{align*}
A_v^\Imc(a) = {} &
	\begin{cases}
	  1 & \text{if } v\in\Vmc \\
	  2/3 & \text{otherwise}
	\end{cases} \\
{A'_v}^\Imc(a) = {} &
	\begin{cases}
	  1 & \text{if } v\notin\Vmc \\
	  2/3 & \text{otherwise}
	\end{cases}
\end{align*}
It is easy to see that this interpretation satisfies $\Amc_\varphi$ and all the axioms in lines \eqref{tbox:one}--\eqref{tbox:three} of 
$\Tmc_\varphi$. Since \Vmc is a model of $\varphi$, for every 3-clause $c\in\varphi$, there exists a literal $\ell$ in $c$ that \Vmc maps 
to true. By construction, $(\rho(\ell))^\Imc(a)=1$, and hence, $\Imc_\varphi$ is also a model of these axioms. Hence $\Omc_\varphi$
is consistent.

Conversely, let $\Imc=(\Delta^\Imc,\cdot^\Imc)$ be a model of $\Omc_\varphi$ and let $\delta=a^\Imc$. We construct the propositional
valuation $\Vmc_\Imc$ as follows: for every $v\in\var(\varphi)$, $v\in\Vmc$ iff $A_v^\Imc(\delta)=1$. By construction, for every
$v\notin\Vmc$ it holds that ${A'_v}^\Imc(\delta)=1$. Since \Imc is a model of $\Tmc_\varphi$, it satisfies the axioms in
line \eqref{tbox:four}, and hence for each clause $c$, there is a literal $\ell$ such that $\rho(\ell)^\Imc(\delta)=1$. This is true iff
$\Vmc_\Imc$ makes $\ell$ true, and in particular the whole clause $c$ true. Hence $\Vmc_\Imc$ satisfies $\varphi$.
\end{proof}
\end{document}